\documentclass[11pt]{article}
\usepackage{jair}
\usepackage{algorithm}
\usepackage{algpseudocode}
\usepackage{theapa}
\usepackage{subfigure}
\usepackage{latexsym}
\usepackage{amsmath}
\usepackage{amssymb}
\usepackage{QED}
\usepackage[]{graphicx}
\graphicspath{{pictures/}}
 
\jairheading{31}{2008}{497-542}{08/07}{03/08}
\ShortHeadings{Exploiting Subgraph Structure in Multi-Robot Path Planning}
{Ryan}
\firstpageno{497}
\newcommand{\Enter}{\textsc{Enter}}
\newcommand{\Exit}{\textsc{Exit}}
\newcommand{\Terminate}{\textsc{Terminate}}

\newcommand{\ResolveExit}{\textsc{ResolveExit}}
\newcommand{\ResolveEnter}{\textsc{ResolveEnter}}
\newcommand{\ResolveTerminate}{\textsc{ResolveTerminate}}
\newcommand{\e}[1]{\textbf{#1}}

\newcommand{\setof}[2]{\ensuremath{\left\{ #1\hspace{1ex}|\hspace{1ex}#2 \right\}}}

\newcommand{\seq}[1]{\ensuremath{\left< #1 \right>}}
\newcommand{\simg}[1]{\ensuremath{\underset{#1}{\sim}}}
\newcommand{\oplusg}[1]{\ensuremath{\underset{#1}{\oplus}}}
\newcommand{\ominusg}[1]{\ensuremath{\underset{#1}{\ominus}}}
\newcommand{\conf}[2]{\ensuremath{\simg{#1}\hspace{-0.5ex}(#2)}}
\newcommand{\perm}[2]{\ensuremath{{}^{#1}\mathrm{P}_{#2}}}

\newcommand{\Pl}{\ensuremath{\mathcal{P}}}
\newcommand{\Fail}{\State \textbf{fail}} 
\newcommand{\Ret}[1]{\State \textbf{return}~#1}
\newcommand{\Choose}[1]{\State \textbf{choose}~#1}

\newtheorem{lemma}{Lemma}
\newtheorem{theorem}{Theorem}
\newtheorem{proposition}{Proposition}

\begin{document}

\title{Exploiting Subgraph Structure in \\
Multi-Robot Path Planning}

\author{\name Malcolm R. K. Ryan \email malcolmr@cse.unsw.edu.au \\
       \addr ARC Centre of Excellence for Autonomous Systems\\
       University of New South Wales, Australia}


\maketitle
 
\begin{abstract}
Multi-robot path planning is difficult due to the combinatorial explosion of 
the search space with every new robot added. Complete search of the combined 
state-space soon becomes intractable. In this paper we present a novel form of 
abstraction that allows us to plan much more efficiently. The key to 
this abstraction is the partitioning of the map into subgraphs of known 
structure with entry and exit restrictions which we can represent compactly. Planning then becomes a search in the much smaller space of subgraph configurations. Once an 
abstract plan is found, it can be quickly resolved into a correct (but possibly
sub-optimal) concrete plan without the need for further search. We prove that 
this technique is sound and complete and demonstrate its practical 
effectiveness on a real map.

A contending solution, prioritised planning, is also evaluated and shown to have
similar performance albeit at the cost of completeness. The two approaches are
not necessarily conflicting; we demonstrate how they can be combined into a
single algorithm which outperforms either approach alone.
\end{abstract}

\section{Introduction}
\label{intro}

There are many scenarios which require large groups of robots to 
navigate around a shared environment. Examples include: delivery robots in an 
office \cite{hada2001mmr}, a warehouse \cite{everett1994rwi}, a shipping yard \cite{alami1998mrc}, or a mine \cite{alarie2002oss}; or even virtual armies in a computer 
wargame \cite{buro2004rga}. In each case we have many robots with independent goals which must traverse a shared environment without colliding with one another. When planning 
a path for just a single robot we can usually consider the rest of the world to 
be static, so that the world can be represented by a graph called a 
\emph{road-map}. The path-planning problem then amounts to finding a path in the 
road-map, for which reasonably efficient algorithms exist. However, in a 
multi-robot scenario the world is not static. We must not only avoid collisions 
with obstacles, but also with other robots. 

Centralised methods \cite{barraquand91}, which treat the robots as a single 
composite entity, scale poorly as the number of robots increases. Decoupled 
methods \cite{lavalle98,erdmann86}, which first plan for each 
robot independently then resolve conflicts afterwards, prove to be much faster 
but are incomplete because many problems require robots to deliberately detour 
from their optimal path in order to let another robot pass. Even if a priority 
ordering is used \cite{vandenberg05}, requiring low priority robots to plan to 
avoid high-priority robots, problems can be found which cannot be solved with 
any priority ordering.

In realistic maps there are common structures such as roads, corridors and open spaces which produce particular topological features in the map which constrain the possible interactions of robots. In a long narrow corridor, for instance, it may be impossible for one robot to overtake another and so robots must enter and exit in a first-in/first-out order. On the other hand, a large open space may permit many robots to pass through it simultaneously without collision. 

We can characterise these features as particular kinds of subgraphs occurring in the road-map. If we can decompose a map into a collection of such simple subgraphs, then we can build plans hierarchically, first planning the movements from one subgraph to another, then using special-purpose planners to build paths within each subgraph.

In this paper we propose such an abstraction. We limit ourselves to considering an
homogeneous group of robots navigating using a shared road-map. We identify
particular kinds of subgraphs in this road-map which place known constraints on
the ordering of robots that pass through them. We use these constraints to
make efficient planning algorithms for traversing each kind of subgraph, and we
combine these local planners into a hierarchical planner for solving arbitrary
problems. 

This abstraction can be used to implement both centralised and prioritised planners, and we demonstrate both in this paper. Unlike most heuristic abstractions, this method is sound and complete. That is, when used with a centralised search it is guaranteed to find a correct plan if and only if one exists. This guarantee cannot be made when prioritised search is used, however the two-stage planning process means that a prioritised planner with the abstraction can often find plans that would not be available to it otherwise. Experimental investigation shows that this approach is most effective in maps with only sparsely connected graph representations. 

\section{Problem Formulation}

We assume for this work that we are provided with a road-map in the form of a 
graph $G = (V,E)$ representing the connectivity of free space for a single 
robot moving around the world (e.g. a vertical cell decomposition or a 
visibility graph, \citeR{lavalle06}). We also have a set of robots $R = \{r_1, 
\ldots, r_k\}$ which we shall consider to be homogeneous, so a single map 
suffices for them all. We shall assume all starting locations and goals lie on
this road-map.

Further, we shall assume that the map is constructed so that collisions only
occur when one robot is entering a vertex $v$ at the same time as another robot
is occupying, entering or leaving this vertex. Robots occupying other vertices in
the map do not affect this movement.  With appropriate levels of underlying
control these assumptions can be satisfied for most real-world problems.

A simple centralised approach to computing a plan proceeds as follows: First, 
initialise every robot at its starting position, then select a robot and move 
it to a neighbouring vertex, checking first that no other robot is currently 
occupying that vertex. Continue in this fashion, selecting and moving one of
the robots at each step until each is at its goal. Pseudocode for this process is 
shown in Algorithm~\ref{alg:naive}. The code is presented as a non-deterministic 
algorithm, with choice points indicated by the \textbf{choose} operator, and 
backtracking required when the \textbf{fail} command is encountered. In practice, a 
search algorithm such as depth-first, breadth-first or A* search is necessary to
evaluate all the alternative paths it presents.

This algorithm does a complete search of the composite space $G^k = G \times G 
\times \dots \times G$, for $k = |R|$ robots. After eliminating vertices which 
represent collisions between robots, the size of the composite graph is given 
by: 
\begin{align*}
\left|V(G^k)\right| & = \perm{n}{k} \\
                           & = \frac{n!}{(n-k)!} \\
\left|E(G^k)\right| & = k \left|E(G)\right| \perm{(n-2)}{(k-1)} \\
                            & = k \left|E(G)\right|\frac{(n-2)!}{(n-k-1)!}
\end{align*}
where $n = 
|V(G)|$ and $k = |R|$. The running time of this algorithm will depend on the 
search algorithm used, but it can be expected to be very long for moderately 
large values of $n$ and $k$. 

\begin{algorithm}[bt]
\caption{A simple centralised planning algorithm.}
\label{alg:naive}
\small
\begin{algorithmic}[1]
\Function{Plan}{$G, a, b$}
\Comment{Build a plan from $a$ to $b$ in graph $G$.}
  \If {$a = b$}
    \Ret{\seq{}}
    \Comment{Nothing to do.}
  \EndIf
  \Choose{$r \in R$}
  \Comment{Choose a robot.}
  \State \textbf{select} $v_f : a[v_f] = r$
  \Comment{Find its location.}
  \Choose{$v_t \in \setof{v}{(v_f, v) \in G}$ }
  \Comment{Choose an edge.}
  \If{$a[v_t] \neq \Box$}
    \Fail
    \Comment{The destination is occupied; backtrack.}
  \Else
  \State $a[v_f] \gets \Box$
  \Comment{Move the robot from $v_f$ to $v_t$.}
  \State $a[v_t] \gets r$
  \Ret $(r, v_f, v_t).$\Call{Plan}{$G, a, b$}
  \Comment{Recurse.}
  \EndIf
\EndFunction
\end{algorithmic}
\end{algorithm} 

\section{Subgraph Abstraction}

\begin{figure}
\begin{center}
\includegraphics[width=\textwidth]{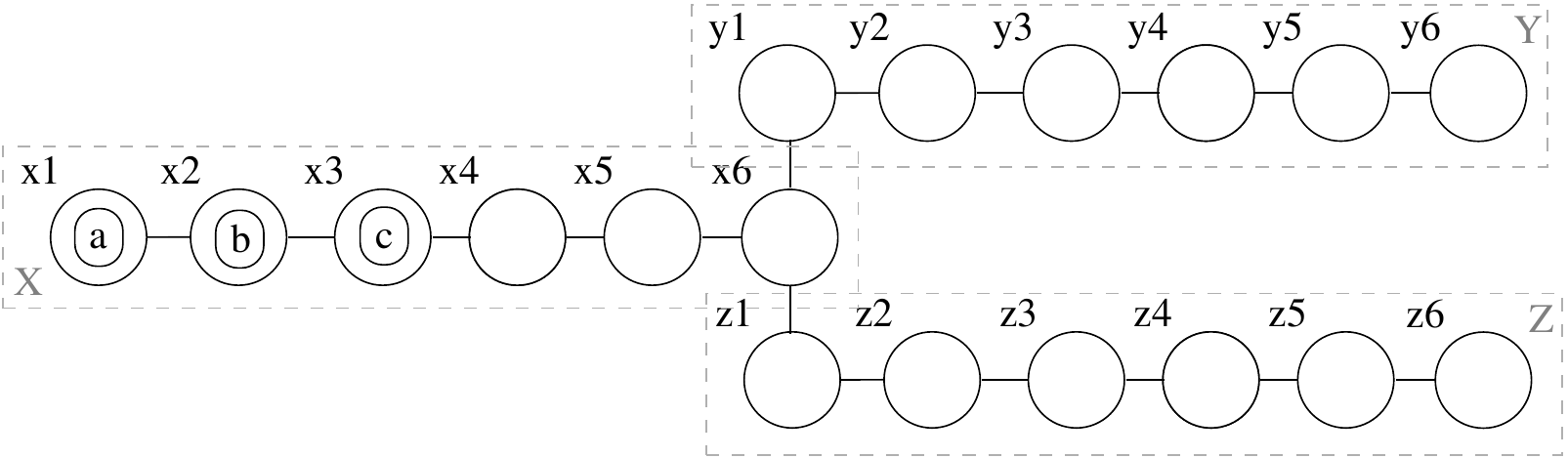}
\end{center}
\caption{A planning problem illustrating the use of subgraphs.}
\label{fig:stack_problem}
\end{figure}

Consider the problem shown in Figure~\ref{fig:stack_problem}. This road-map contains 18 vertices and 17 edges, and there are 3 robots to plan for. So, according to the above 
formulae, the composite graph has $18! / 15! = 4896$ vertices and $3  \times17 
\times 16! / 14! = 12240$ edges. A small map has expanded into a large search 
problem. But to a human mind it is obvious that a lot of these arrangements are 
equivalent. What is important is not the exact positions of the robots, but 
their ordering.

Consider the subgraph labeled $X$. We recognise this subgraph as a
\emph{stack}. That is, robots can only move in and out of this subgraph in a
last-in-first-out (LIFO) order. Robots inside the stack cannot change their
order without exiting and re-entering the stack. So if our goal is to reverse
the order of robots in X, we know immediately that this cannot be done without
moving all the robots out of the stack and then have them re-enter in the
opposite order. Once the robots are in the right order, rearranging them into
the right positions is trivial. Thus we can make a distinction between the \emph{arrangement} of the robots (in which we specify exactly which vertex each robot occupies) and the \emph{configuration} of the stack (in which we are only interested in their order).

Now $X$ has 6 vertices, so when there are $m$ robots in the stack,
there are $\perm{6}{m} = 6! / (6 - m)!$ possible arrangements. So the total number of 
arrangements is:
\begin{align*} 
\perm{6}{3} + 3 \times \perm{6}{2} + 3 \times \perm{6}{1} + \perm{6}{0} & = 120 + 3 \times 30 + 3 \times 6 + 1\\
& = 229
\end{align*}
In terms of deciding whether a robot can leave the stack, however, all we need 
to know is their order. So we need only represent $3! + 3 \times 2! + 3 \times 
1! + 1 = 16$ different configurations of the stack.

Subgraphs $Y$ and $Z$ are also stacks. Applying this analysis to all three, we
find that we can represent the abstract state space with only 60 different
states, and 144 possible transitions between states (moving the top-most robot off
one stack onto another). This is dramatically smaller than the composite map
space above.

A stack is a very simple kind of subgraph and we will need a larger collection
of canonical subgraphs to represent realistic problems. The key features we are
looking for are as follows:
\begin{enumerate}
  \item Computing transitions to and from the subgraph does not require
  knowledge of the exact arrangement of robots within the subgraph, only some
  more abstract configuration (in this case, their order).
  \item If two arrangements of robots share the same configuration, then
  transforming one into the other can be done easily without search,
  \item Therefore planning need only be done in the configuration space, which
  is significantly smaller.
\end{enumerate}
Later we will introduce three more subgraph types -- cliques, halls and rings --
which also share these properties and which are readily found in realistic
planning problems. But first we need to formalise the ideas of subgraph
planning.

\section{Definitions}

In this section we outline the concepts we will use later in the paper. A complete formal definition of these terms is provided in the Appendix, along with a proof of soundness and completeness of the subgraph planning process.

Given a map represented by a graph $G$ we partition it into a set of disjoint subgraphs ${S_1, \ldots, S_m}$. These subgraphs are \emph{induced}, i.e. an edge exists between two vertices in a subgraph if and only if it also exists in $G$.

An \emph{arrangement} $a$ of robots in $G$ is a 1-to-1 partial function $a : V(G) \rightarrow R$, which represents the locations of robots within $G$. If robot $r$ is in vertex $v$, we write $a(v) = r$. We can also speak of the arrangement of robots within a subgraph $S$. We will denote arrangements by the lowercase roman letters $a, b$

A \emph{configuration} of a subgraph $S$ is a set of equivalent arrangements of robots within $S$. Two arrangements are equivalent if there exists a plan to move robots from one to the other without any robots leaving the subgraph. We will denote a configuration of subgraph $S_x$ by $c_x$. The configuration of the whole map can then be represented as a tuple of subgraph configurations $\gamma = (c_1, \ldots, c_m)$.

There are two operators $\oplus$ and $\ominus$ which operate on configurations, representing a robot entering and leaving the subgraph respectively. When a robot $r$ moves between two subgraphs $S_x$ and $S_y$ their configurations change depending on the identity of the edge $(u,v)$ on which the robot traveled. We write:
\begin{align*}
&c_x' \in c_x \ominus (r,u), \\
&c_y' \in c_y \oplus (r,v)
\end{align*}
In complex subgraphs it is possible for such a transition to result in several possible configurations, so the operators $\oplus$ and $\ominus$ return sets. It is also possible that a transition is impossible from a particular configuration, in which case the operation returns the empty set.

An abstract plan $\Pi$ can now be defined as a sequence of transitions $\Sigma$ with intermediate configurations $\Gamma$. For every abstract plan between two arrangements there exists at least one corresponding concrete plan, and vice versa. All the subgraph transitions in the concrete plan must also exist in the abstract plan. The equivalence of arrangements in a configuration then guarantees the existence of the intermediate steps. See the Appendix for a complete proof.

\begin{algorithm}[h!]
\caption{Planning with subgraph abstraction.}
\label{alg:subgraph}
\small
\begin{algorithmic}[1]
\Function{Plan}{$G, \Pl, R, a, b$}
\Comment{Build a plan from $a$ to $b$ in $G$ using partition $\Pl$.}
  \State $\alpha \gets \gamma(a)$
  \Comment{Get the initial configuration.}
  \State $\beta \gets \gamma(b)$
  \Comment{Get the final configuration.}
  \State $\Pi \gets $ \Call{AbstractPlan}{$G, \Pl, R, \alpha, \beta$}
  \Comment{Build the abstract plan.}
  \State $P \gets $ \Call{Resolve}{$G, \Pl, \Pi, a, b$}
  \Comment{Resolve to a concrete plan.}
  \Ret{$P$}
\EndFunction
\end{algorithmic}
\begin{algorithmic}[1]
\Statex
\Function{AbstractPlan}{$G, \Pl, R, \alpha, \beta$}
\Comment{Build an abstract plan from $\alpha$ to $\beta$ in $G$ using $\Pl$.}
  \If {$\alpha = \beta$}
    \Ret{$(\seq{\beta}, \seq{})$}
    \Comment{Done.}
  \EndIf  
  \State $(c_1, \ldots, c_m) = \alpha$
  \Choose{$r \in R$}
  \Comment{Choose a robot.}
  \State{\textbf{select} $x : r \in range(c_x)$}
  \Comment{Find the subgraph it occupies.}
  \Choose{$S_y \in \Pl : (S_x, S_y) \in X$}
  \Comment{Choose a neighbouring subgraph.}
  \Choose{$(u, v) \in E(G) : u \in S_x, v \in S_y$}
  \Comment{Choose a connecting edge.}
  \Choose{$c_x' \in c_x \ominus (r, u)$}
  \Comment{Choose the resulting configurations of $S_x$ and $S_y$.}
  \Choose{$c_y' \in c_y \oplus (r, v)$}
  \State $\gamma \gets (c_1, \ldots, c_x', \ldots, c_y', \ldots, c_m)$
  \Comment{Construct the new configuration tuple.}
  \State $(\Gamma, \Sigma) \gets $ \Call{AbstractPlan}{$G, \Pl, R, \gamma,  \beta$}
  \Comment{Recurse.}
  \State $\Gamma' \gets \gamma . \Gamma$
  \State $\Sigma' \gets (r, u, v) . \Sigma$
  \Ret{$(\Gamma', \Sigma')$}
\EndFunction
\Statex
\end{algorithmic}
\begin{algorithmic}[1]
\Function{Resolve}{$G, \Pl, \Pi, a, b$}
\Comment{Resolve the abstract plan into a concrete plan.}
  \State $\Pi = (\Gamma, \Sigma)$
  \State $\Gamma = \seq{\gamma_0, \ldots, \gamma_l}$
  \State $\Sigma = \seq{s_1, \ldots, s_l}$
  \State $P \gets \seq{~}$
  \State $a^0 \gets a$
  \For {$i = 0 \ldots (l-1)$}
    \State $(r, u, v) = s_{i+1}$
    \Comment{The next transition.}
    \State $(c_1', \ldots, c_m') = \gamma_{i+1}$
    \Comment{The target configurations.}
    \State \textbf{find} $S_x : u \in S_x$
    \State \textbf{find} $S_y : v \in S_y$
    \State $a^i_z \gets a^i / S_z, \forall z = 1 \ldots m$
    \State $(P^i_x, b^i_x) \gets S_x.$\Call{ResolveExit}{$a^i_x, r, u, c_x'$}
    \Comment{Rearrange $S_x$ to let robot $r$ exit.}
    \State $(P^i_y, b^i_y) \gets S_y.$\Call{ResolveEnter}{$a^i_y, r, v, c_y'$}
    \Comment{Rearrange $S_y$ to let robot $r$ enter.}
    \State $P \gets P . (P^i_x || P^i_y)$
    \State $b^i = a^i_1 \otimes \ldots \otimes b^i_x \otimes \ldots \otimes b^i_y
   	\otimes \ldots \otimes a^i_m$ 
    \State $a^{i+1} \gets s_{i+1}(b^i)$
    \State $P \gets P .  . \seq{s_{i+1}}$
    \Comment{Add the transition.}
  \EndFor
  \For {$z = 1 \ldots m$}
    \State $T_z \gets S_z.$\Call{ResolveTerminate}{$a^l / S_z, b / S_z$}
    \Comment{Rearrange $S_z$ into its final arrangement.}
  \EndFor
  \State $P \gets P . (T_1 || \ldots || T_m)$
  \Ret $P$
\EndFunction
\end{algorithmic}
\end{algorithm} 

\section{Subgraph Planning}

We can now construct a planning algorithm which searches the space of abstract 
plans (Algorithm \ref{alg:subgraph}). The procedure is much the same as before. 
First we compute the configuration tuple for the initial arrangement. Then we 
extend the plan one step at a time. Each step consists of selecting a robot $r$ 
and moving it from the subgraph it currently occupies $S_x$ to a neighbouring 
subgraph $S_y$ in the reduced graph $X$, along a connecting edge $(u,v)$. 

This transition is only possible if the plan-step $(s, (u,v))$ is applicable. If
it is, it may result in a number of different configurations in the subgraph entered. We need to choose one to create the configuration tuple for the next step. Both the applicability test and the selection of the subsequent configurations are
performed in lines 10-11 of \textsc{AbstractPlan}.

The abstract plan is extended step by step in this fashion until it reaches a
configuration tuple which matches the goal arrangement. The resulting abstract
plan is then resolved into a concrete plan. For each transition in the abstract
plan we build two short concrete plans -- one to move the robot to the outgoing
vertex of the transition, and one to make sure the incoming vertex is clear and the subgraph is appropriately arranged to create the subsequent configuration.
Since these two plans are on separate subgraphs, they can be combined in
parallel. The final step is to rearrange the robots into the goal arrangement.
Again, this can be done in parallel on each of the subgraphs.

\textsc{AbstractPlan} has been written as a non-deterministic program,
including choice-points. A search algorithm such as breadth-first or depth-first
search is needed to examine each possible set of choices in some ordered
fashion. If this search is complete then an abstract plan is guaranteed to be
found, if one exists and so by the theorem above this planning algorithm is both
sound and complete. Note that the resolution phase of the planner is entirely
deterministic, so no further search is needed once an abstract plan is found.

\subsection{Subgraph Methods}

The efficiency of this algorithm relies on being able to compute several functions
without a lot of search:
\begin{itemize}
  \item{\textsc{Exit}} To compute $c \ominus (r, u)$, testing if it is possible
  for a robot to exit the subgraph and determining the resulting configuration(s).
  \item{\textsc{Enter}} To compute $c \oplus (r, v)$, testing if it is possible
  for a robot to enter the subgraph and determining the resulting configuration(s). 
  \item{\textsc{Terminate}} To compute $b / S \in c$, testing if it is possible
  for the robots in the subgraph to move to their terminating positions.
  \item{\textsc{ResolveExit}} To build a plan rearranging robots in a subgraph to
  allow one to exit.
  \item{\textsc{ResolveEnter}} To build a plan rearranging robots in a subgraph to
  allow one to enter.
  \item{\textsc{ResolveTerminate}} To build a plan rearranging robots in a subgraph
  into their terminating positions.
\end{itemize}

The key to efficient subgraph planning is to carefully constrain the allowed 
structure of the subgraphs in our partition, so these functions are simple to 
implement and do not require expensive search. The advantage of this approach 
is that each of these functions can always be computed based only on the 
arrangement of other robots within that particular subgraph, not relying on 
the positions of robots elsewhere.

\section{Subgraph Structures}
\begin{figure}[tb]
\begin{center}
\subfigure[A stack]{\label{fig:stack}\includegraphics[]{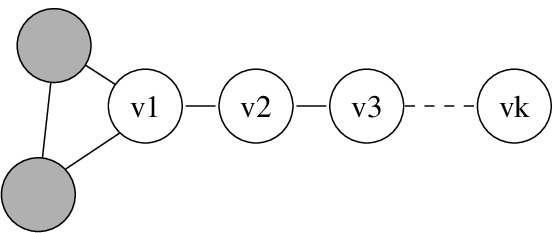}} 
\subfigure[A hall]{\label{fig:hall}\includegraphics[]{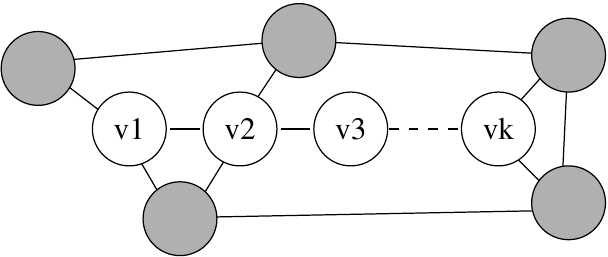}} 
\subfigure[A clique]{\label{fig:clique}\includegraphics[]{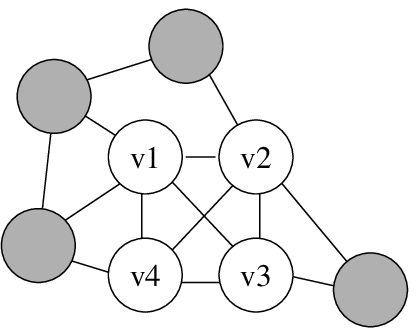}}\quad
\subfigure[A ring]{\label{fig:ring}\includegraphics[]{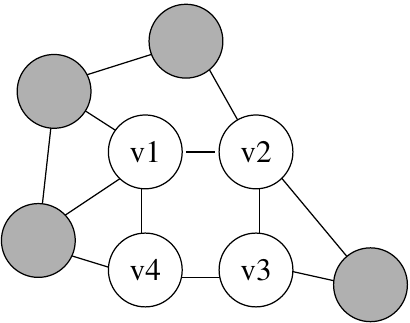}}
\caption{Examples of the four different subgraph structures.}
\label{fig:subgraphs}
\end{center}
\end{figure} 
The key to this process is therefore in the selection of subgraph types. These
abstractions need to be chosen such that:
\begin{enumerate}
  \item They are commonly occurring in real road-maps,
  \item They are easy to detect and extract from a road-map,
  \item They abstract a large portion of the search space,
  \item Computing the legality of transitions is fast, sound and complete,
  \item Resolving an abstract plan into a concrete sequence of movements is efficient.
\end{enumerate}
In this paper we present four subgraph types: stacks, halls, cliques and rings,
which satisfy these requirements. In the following analysis, let $n$ be the the number of vertices in the subgraph and $k$ be the number of robots occupying the subgraph \emph{before} the action takes place.

\subsection{Stacks}
\label{sec:stacks}
A stack (Figure~\ref{fig:stack}) represents a narrow dead-end corridor in the road-map.
It has only one exit and it is too narrow for robots to pass one another, so robots must enter and leave in a last-in-first-out order. It is one of the simplest subgraphs and does not occur often in real maps, but it serves as an easy illustration of the subgraph
methods. Formally it consists of a chain of vertices, each linked only to its predecessor
and its successor. Only the vertex at one end of the chain, called the \emph{head}, is
connected to other subgraphs so all entrances and exits happen there.

A configuration of a stack corresponds to an ordering of the robots that reside in
it, from the head down. Robots in the stack cannot pass each other, and so the
ordering cannot be changed without the robots exiting and re-entering the stack.

\subsubsection{\Enter{}}

A robot can always enter the stack as long as the stack is not full. Only one
new configuration is created, adding the robot to the front of the ordering.
This computation can be done in $O(1)$ time.

\subsubsection{\Exit{}}

A robot can exit the stack only if it is the top robot in the ordering. Only one
new configuration is created, removing the robot from the ordering. This
computation can also be done in $O(1)$ time.

\subsubsection{\Terminate{}}

To determine whether termination is possible, we need to check if the order of robots in
the current configuration is the same as that in the terminating arrangement.
This operation takes $O(k)$ time.

\subsubsection{\ResolveEnter{}}

Rearranging robots inside the stack is simple since we know that the ordering
is constant. To vacate the top of the stack (the only possible entrance point)
we move robots deeper into the stack (as necessary). There is guaranteed to be
room, since entering a full stack is not permitted. At worst this takes $O(k)$
time.

\subsubsection{\ResolveExit{}}

When a robot exits the stack, the abstract planner has already determined that
it is the first robot in the stack with no others between it and the head
vertex. It can simply move up the stack to the head, and then out. No other
robots need to be moved. At worst this takes $O(n)$ time.

\subsubsection{\ResolveTerminate{}}

Finally, moving robots to their terminating positions can be done in a
top-to-bottom order. If a robot is below its terminating position it can move
upwards without interference. If a robot is above its terminating position,
other robots below may need to be moved lower in order to clear its path. This
approach is sound, since the terminating positions of these robots must be
further down the stack (or else the ordering would be different). This
process has an $O(nk)$ total worst-case running time.

\subsection{Halls}
\label{sec:halls}
A hall is a generalisation of a stack (Figure~\ref{fig:hall}). Like a stack, it is a narrow corridor which does not permit passing, but a hall may have multiple entrances and exits along its length. Formally it consists of a single chain of vertices, each one joined to its predecessor and its successor. There must be no other edges between vertices in the hall, but there may be edges connecting to other subgraphs from any node 
in the hall. Halls are much more commonly occurring structures,  but still maintain the same property as stacks: the robots cannot be reordered without exiting and re-entering. Thus, as with stacks, the configuration of a  hall corresponds to the order of the robots occupying it, from one end of the hall to the other.

\subsubsection{\Enter{}}
\begin{figure}[tb]
\begin{center}
\includegraphics[]{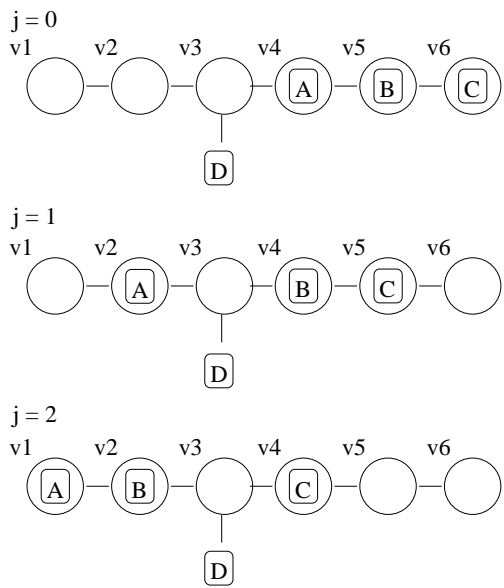}
\caption{Example of entering a hall subgraph, with $k = 3$, $n = 6$ and $i = 3$.
Robot D can enter at three possible sequence positions $j$ = 0, 1 or 2 but not
at $j = 3$.}
\label{fig:enter_hall}
\end{center}
\end{figure} 

A robot can enter a hall as long as it is not full. The configurations generated
by that entrance depend on three factors: 1) The size of the hall $n$, 2) The
number of robots already in the hall $k$, and 3) The index $i$ of the vertex at
which it enters (ranging from 1 to $n$).

Figure~\ref{fig:enter_hall} shows how entering a hall can result in several
different configurations. It is a matter of how the robots already in the hall
are arranged, to the left and right of the entrance, before the entering robot
moves in. If there is enough space in the hall on either side of the entrance
vertex, then the new robot can be inserted at any point in the ordering. But if
space is limited (as in the example) then it may not be possible to move all the
robots to one side or another, limiting the possible insertion points.

Given the three variables $k, n, i$ above, we can compute the maximum and
minimum insertion points as:
\begin{align*}
j &\leq min(i-1, k) \\
j &\geq max(0, k - (n - i))
\end{align*}

Creating a new configuration is then just a matter of inserting the new robot
into the ordering at the appropriate point. Since the list of robots needs to be
copied in order to do this, it takes $O(k)$ time for each new configuration.

\subsubsection{\Exit{}}

Whether a robot can exit a hall via a given edge again depends on several
factors: 1) the size of the hall $n$, 2) the number of robots in the hall $k$,
3) the index $i$ of the vertex from which it exits (from 1 to $k$), 3) the index $j$
of the robot in the ordering (from 1 to $k$). Exit is possible if:
\[
j \leq i \leq n - (k-j)
\]
If exit is possible there is one resulting configuration: the previous ordering
with the robot removed. This takes $O(k)$ time to compute.

\subsubsection{\Terminate{}}

Checking termination is the same for halls as with stacks, we just have to test
that the order of robots in the final arrangement matches the current
configuration. This can be done in $O(k)$ time for $k$ robots in the hall.

\subsubsection{\ResolveEnter{}}

To resolve an entrance to a hall we need to know which of the subsequent
configurations we are aiming to generate, so we know the proper insertion point 
for the entering robot. The robots before the insertion point are shuffled in
one direction so that they are on one side of the entry vertex, and the rest to
the other side. At worst this will take $O(nk)$ time.

\subsubsection{\ResolveExit{}}

Resolving an exit involves moving the robot up or down the hall to the exit
vertex, shuffling any other robots that are in the way. In the worst case, in
which all the robots shuffle from one end of the hall to the other, this takes
$O(nk)$ time.

\subsubsection{\ResolveTerminate{}}

\ResolveTerminate{} for a hall is identical to that for a stack, described 
above.

\subsection{Cliques}
\label{sec:cliques}

A clique (Figure~\ref{fig:clique}) represents a large open area in the map with many exit points (vertices) around its perimeter. Robots can cross directly from any vertex to another, and as long as the clique is not full, other robots inside can be shuffled out of the way to allow this to happen.

Formally a clique is a totally connected subgraph. Cliques have
quite different properties to halls and stacks. As long as there is at least one
empty vertex in a clique, it is possible to rearrange it arbitrarily. So a
configuration of a clique, in this circumstance, is just the set of robots it
contains. 

However there are a special set of configurations in which the clique is
\emph{locked}. This occurs when the number of robots in the clique equals the
number of vertices. Then it is impossible for the clique to be rearranged. A 
configuration of a locked clique has to explicitly record the position of each
robot. 

\subsubsection{\Enter{}}

A clique can always be entered so long as it is not full. If the clique has more
than one vacant vertex, then there is a single new configuration with the
entering robot added to the set of occupants. If the clique has only one space
remaining, then the entering robot locks the clique. In theory, at this point it
is necessary to make a new configuration for every possible arrangement of the
occupying robots (with the entering robot always in the vertex it enters). 

In practice, it is more efficient to create just a single ``locked'' configuration which records the locking robot and its vertex, and leaves the other positions unspecified. Any permutation of the other robots is possible, so the exact details of the configuration need not be pinned down until the next action (either \Exit{} or \Terminate{}) requires them to be. This is a form of \emph{least commitment}, and it can significantly reduce the branching factor of our search.

Performing this test and creating the new configuration takes $O(k)$ time for
$k$ robots in the clique.

\subsubsection{\Exit{}}

If the clique is unlocked then any robot can exit from any vertex and the new 
configuration is created by simply removing the robot from the set of occupants.

If the clique is locked then a robot can only exit from the specific vertex that
it occupies. The resulting configuration is unlocked and the exact locations of
the robots can be discarded.

In the least-commitment version, the locking robot is constrained to exit from
its vertex and every other robot can exit from any vertex \emph{except} the one
occupied by the locking robot. 

Performing this test and creating the new configuration takes $O(k)$ time for
$k$ robots in the clique.

\subsubsection{\Terminate{}}

With an unlocked configuration, checking for termination simply consists of
making sure that all (and only) the required occupants are in the clique. For a
locked configuration the robots must all be in their terminating positions (as
there is no possibility of rearranging them). In the least-commitment version
just the locking robot must be in its terminating vertex. We can then assume that
all the other robots are also in their places (thus committing to a choice of
configuration that we delayed earlier).

Performing this test takes $O(k)$ time for $k$ robots in the clique.

\subsubsection{\ResolveEnter{}}

If the entrance vertex is occupied when a robot wishes to enter then we can
simply move the occupant directly to another vacant vertex in the clique, since
every vertex is connected to every other.

If we are using least commitment and the entering robot locks the clique
then we need to look ahead in the plan to see the next action involving
this clique. If it is an exit transition then we need to move the exiting robot
to the exit vertex now (\emph{before} the clique is locked). If there is no
subsequent exit, meaning the robots will be terminating in this clique, then we
need to rearrange them into their terminating positions at this point.

If we amortise the cost of any rearrangements over the subsequent call to 
\ResolveExit{} or \ResolveTerminate{} we can treat this operation as taking
$O(1)$ time.

\subsubsection{\ResolveExit{}}

If the clique is full at the time of exit then we can assume that the exiting
robot is already at its exit vertex and nothing needs to be done. On the other
hand, if the clique is not full it may be that the robot is not at its exit
vertex. It must be moved there. If the exit vertex is already occupied by
another robot, it can be moved into another unoccupied vertex. Both these
movements can be done directly, as the clique is totally connected.
This operation takes $O(1)$ time.

\subsubsection{\ResolveTerminate{}}

If the clique is locked then we can assume that the robots have already been
appropriately arranged into their terminal positions and no further work needs
to be done. Otherwise the robots may need to be rearranged. A simple way to do
this is to proceed as follows: for each robot that is out of place, first vacate
its terminating position by moving any occupant to another unoccupied vertex,
then move the terminating robot into the vertex. Once a robot has been moved in
this way it will not have to move again, so this process is correct but it may
produce longer plans than necessary. The upside is that it takes only $O(n)$ time.

\subsection{Rings}
\label{sec:rings}

A ring (Figure~\ref{fig:ring}) resembles a hall with its ends connected. 
Formally, it is a subgraph $S$ with vertices $V(S) = \{v_1, \ldots, v_n\}$ and induced edges $E(S)$ satisfying:
\[
(v_i, v_j) \in E(S) \text{~iff~} | i - j | \equiv 1~(\text{mod~}n)
\]
As with a hall, ordering is important in a ring. Robots in the ring cannot pass
one another and so cannot re-order themselves. They can, however, rotate their
ordering (provided that the ring is not full). Thus in a ring of size 4 or more,
the sequence $\seq{r_1, r_2, r_3}$ is equivalent to $\seq{r_3, r_1, r_2}$ but
not to $\seq{r_2, r_1, r_3}$. Equivalent sequences represent the same configuration.

Like cliques, rings are \emph{locked} when they are full. A locked ring cannot
be rotated, so in a ring of size three the sequences $\seq{r_1, r_2, r_3}$
and $\seq{r_3, r_1, r_2}$ are \emph{not} equivalent. They represent two locked configurations with different properties.

\subsubsection{\Enter{}}

A robot may always enter a ring provided that it is not full. If there are $k$
robots already occupying the ring, then there are $k$ possible configurations
that can result (or one if $k$ is zero), one for each possible insertion point.

If the entering robot locks the ring then we must also record the specific
positions of each robot in the ring. This will still only produce $k$ different
configurations because the robots cannot be arbitrarily rearranged, unlike in
cliques. 

It is also possible to produce least-commitment versions of \Enter{} for rings
as with cliques. Again, this can significantly reduce the branching factor of
the search, but the details are more involved than we wish to enter into in this
paper.

This operation takes $O(k)$ time for each new configuration generated.

\subsubsection{\Exit{}}

When the ring is locked a robot can only exit from its recorded position, 
otherwise it can exit from any vertex. The robot is removed from the sequence to
produce the resultant configuration. The new configuration is unlocked and any
position information can be discarded. This can be done in $O(k)$ time for $k$ robots in the ring.

\subsubsection{\Terminate{}}

To check if termination is possible we need to see if the order of robots around
the ring in the terminal arrangement matches that of the current configuration.
If the configuration is not locked then rotations are allowed, otherwise the
match must be exact. This test can be done in $O(k)$ time for $k$ robots in the ring.

\subsubsection{\ResolveEnter{}}

When a robot is about to enter the ring, we need to first rearrange it so that the
the entry vertex is empty and the nearest robots on either side of that vertex
provide the correct insertion point for the subsequent configuration, as
selected in \Enter{} above. This may require shuffling the robots one way or
another, in much the same fashion as in a stack or hall. In the worst case this
will take $O(nk)$ operations for $k$ robots in a ring of $n$ vertices.

\subsubsection{\ResolveExit{}}

If a ring is locked then any robot exiting must already be at its exit position
so nothing needs to be done. Otherwise, in an unlocked ring, the robots may need
to be shuffled around the ring in order to move the robot to its exit. In the
worst case this will take $O(nk)$ operations for $k$ robots in a ring of $n$ vertices.

\subsubsection{\ResolveTerminate{}}

If a ring is locked then all the robots must already be in their terminating
positions; this is guaranteed by the abstract planner. Otherwise they will need
to be rotated into the correct positions. Once one robot has been moved to its
correct vertex, the rest of the ring can be treated as a stack and the
\ResolveTerminate{} method described above can be used, with $O(nk)$ worst case
running time for $k$ robots in a ring of $n$ vertices.

\subsection{Summary}

Of these four subgraphs halls and rings are the most powerful. Such subgraphs
are not only common in the structured maps of man-made environments, but can
also be found often in purely random graphs (consider: any shortest path in an
unweighted graph is a hall). Halls, rings and cliques of significant size can be found in many realistic planning problems.

Importantly, these structures are well constrained enough that the six procedures for planning outlined above can all be implemented efficiently and deterministically, 
without the need for any further search. In the cases of the clique and the ring, the
resolution methods we describe sometimes sacrifice path optimality for speed,
but this could be improved by using smarter resolution planners. Since the
resolution stage is only done once, this probably would not have a major effect
on the overall running time of the planner.

\section{Prioritised Planning}

A common solution to the rapid growth of search spaces in multi-robot planning is 
\emph{prioritised planning} \cite{erdmann86,vandenberg05}. In 
this approach we give the robots a fixed priority ordering before we begin. 
Planning is performed in priority ordering: first a plan is built for just the 
robot with highest priority; then a plan for the second highest, such that it
does not interfere with the first; then the third, and so on. Each new plan must 
be constructed so that it does not interfere with the plans before it. An 
example implementation is shown in Algorithm~\ref{alg:naive_priority}. Usually 
there is no backtracking once a plan has been made. This is signified in the 
algorithm by the \textbf{cut} operator in line 8 of \textsc{Plan}.

\begin{algorithm}[bt]
\caption{A simple prioritised planning algorithm.}
\label{alg:naive_priority}
\small
\begin{algorithmic}[1]
\Function{Plan}{$G, a, b$}
  \State $a'[v] \gets \Box, \forall v \in G$
  \Comment{$a'$ is the initial arrangement of robots $r_1 \ldots r_i$.}
  \State $b'[v] \gets \Box, \forall v \in G$  
  \Comment{$b'$ is the final arrangement of robots $r_1 \ldots r_i$.}
  \For{$i = 1 \ldots k$}
     \State $a'[v] = r_i, \mathrm{~for~}v : a[v] = r_i$
     \State $b'[v] = r_i, \mathrm{~for~}v : b[v] = r_i$
     \State $(P, P_i) \gets$\Call{PlanOne}{$G, r_i, \seq{P_1, \ldots, P_{i-1}}, \seq{0, \ldots, 0}, a', b'$}
     \Comment{Build a plan for $r_1 \ldots r_i$.}
     \State \textbf{cut}
     \Comment{Do not backtrack once a plan is found}
  \EndFor  
  \Ret $P$
\EndFunction
\end{algorithmic}
\begin{algorithmic}[1]
\Statex
\Function{PlanOne}{$G, r_i, \seq{P_1, \ldots, P_{i-1}}, \seq{t_1, \ldots, t_{i-1}}, a,
b$}
  \If {$a = b$}
    \Ret{(\seq{}, \seq{})}
    \Comment{Done.}
  \EndIf
  \Choose{$r_j \in R : j \leq i$}
  \Comment{Choose a robot to move.}
  \If{$j = i$}
    \State \textbf{select} $v_f : a[v_f] = r_i$
    \Choose{$v_t \in \setof{v}{(v_f, v) \in G}$ }  
    \Comment{Choose a new action for $r_i$}
  \Else
    \State $(r, v_f, v_t) \gets P_j[t_j]$
    \Comment{Select the old action for $r_j$ from $P_j$}
    \State $t_j \gets t_j + 1$
  \EndIf
  \If{$a[v_t] \neq \Box$}
    \Fail
    \Comment{Backtrack if the destination is occupied.}
  \EndIf
  \State $a[v_f] \gets \Box$
  \Comment{Move the robot.}
  \State $a[v_t] \gets r$
  \State $(P, P_i) \gets $\Call{PlanOne}{$G, r_i, \seq{P_1, \ldots, P_{i-1}},  \seq{t_1, \ldots, t_{i-1}}, a, b$} 
  \Comment{Recurse.}
  \State $P \gets (r_j, v_f, v_t). P$
  \Comment{Add this step to the global plan.}
  \If{$j = i$}
    \State $P_i \gets (r_i, v_f, v_t). P_{r_i}$
    \Comment{Add this step to $r_i$'s plan.}
  \EndIf
  \Ret $(P, P_i)$
\EndFunction
\end{algorithmic}
\end{algorithm} 

\begin{figure}[b!]
\begin{center}
\includegraphics[scale=0.5]{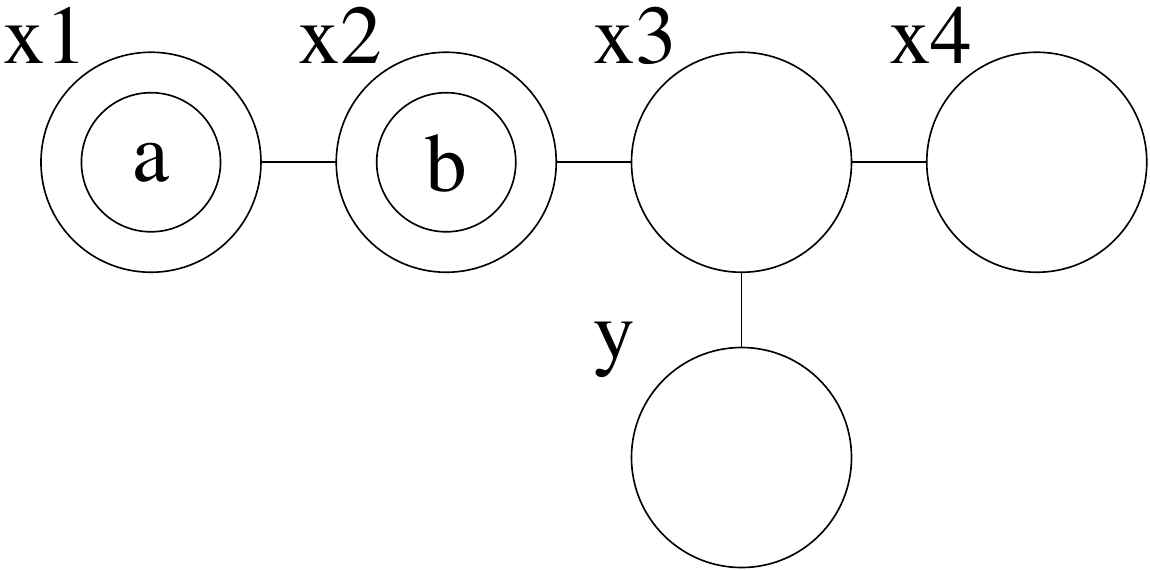}
\end{center}
\caption{A simple planning problem that cannot be solved with naive prioritised
planning. The goal is to swap the positions of robots $a$ and $b$.}
\label{fig:insoluble1}
\end{figure}

Because of this cut, the search is no longer complete. There are problems with
solutions that a prioritised planner cannot find. Figure~\ref{fig:insoluble1} is
an example. Robots $a$ and $b$ wish to change positions. To plan for either
robot on its own is easy; the plan contains just one step. But to plan for both
robots together requires each of them to move out of its way, to the right hand
side of the map so that the other can pass. A prioritised planner which
committed to a one-step plan for either $a$ or $b$ cannot then construct a plan
for the other robot which does not interfere.

This incompleteness is not just a mistake, however. It is the core of what makes
prioritised planning more efficient. The search space has been pruned
significantly by eliminating certain plans from consideration. If there is still
a viable solution within this pruned space (and often there is) then it can be
found much more quickly. In the (hopefully few) cases where it fails, we can
always resort to a complete planner as a backup.

\subsection{Prioritised Subgraph Planning}

Prioritised planning is not strictly a competitor to subgraph planning. In 
fact, prioritised search and the subgraph representation are 
orthogonal ideas, and it is quite possible to use both together. As in 
Algorithm~\ref{alg:naive_priority}, a plan is constructed for each robot
consecutively, but rather than building an entire concrete plan, only the
abstract version is produced, in the fashion of Algorithm~\ref{alg:subgraph} earlier.
Only when compatible abstract plans have been produced for every robot, are they
resolved into a concrete plan. 

As well as adding the advantage of abstraction to prioritised planning, the 
subgraph representation also allows the planner to cover more of the space of possible plans. By delaying resolution until the end, we avoid commitment to concrete choices for a high priority robot which will hamper the planning of later robots.

To illustrate this, let's return to the example in Figure~\ref{fig:insoluble1}
above. If we partition this subgraph so that vertices $\{x_1, x_2, x_3, x_4\}$ are a
hall $X$, then the prioritised subgraph planner can solve the problem. The
abstract plan for the highest priority robot is empty; there is nothing for it
to do as it is already in its goal subgraph. Given this plan, the second highest
priority robot can plan to move from $X$ to $y$ and then back again. This plan
can produce the goal configuration required. Resolving this plan will move the
highest priority robot to $x_4$ and back again as needed, but this plan will be
built by the \textsc{Resolve} methods for halls, and not by search.

\begin{figure}[b!]
\begin{center}
\includegraphics[scale=0.5]{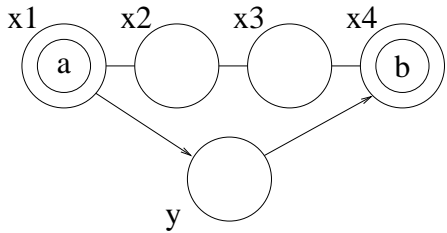}
\end{center}
\caption{A simple planning problem that can be solved with naive prioritised
planning but not with the subgraph abstraction. The goal is to swap the positions of robots $a$ and $b$. With priority ordering $a,b$ the subgraph planner will choose for robot $a$ to remain inside the hall. Robot $b$ is then trapped, because $a$ blocks the only exit to $y$ (note that the edges $(x_1, y)$ and $(y, x_4)$ are directed).}
\label{fig:insoluble2}
\end{figure}

Of course there is no such thing as a free lunch and this example only works if
we choose the right partition. If instead we treat $\{x_1, x_2\}$ as a stack and
$\{x_4, x_3, y\}$ as a separate hall then the prioritised subgraph planner will
not help us. Furthermore there exist problems, such as the one in Figure~\ref{fig:insoluble2} which can be solved by standard prioritised planners but will fail if
we introduce the wrong subgraph abstraction. It is difficult to generate more realistic cases of this problem with small numbers of robots, but as we will see in Section~\ref{sec:exp3} below, they can occur when the number of robots is large. 

\section{Search Complexity}
\label{sec:search}

Let us consider more carefully where the advantages (if any) of the subgraph decomposition lie. Subgraph transitions act as \emph{macro-operators} between one abstract state (set of configurations) to another. There is a long history of planners using macros of one kind or another, and their advantages and disadvantages are well known (see Section\ref{sec:related}). It is widely recognised that macros are advantageous when they reduce the depth of the search, but become a disadvantage when too many macros are created and the branching factor of the search becomes too large. These guidelines also apply to the use of subgraphs. 

A typical search algorithm proceeds as follows: select a plan from the frontier of incomplete plans and create all expansions. Add all the expansions to the frontier and recurse until a complete plan is found. The time taken to complete this search is determined by the number of nodes in the search tree, which is in turn determined by three factors:
\begin{enumerate}
\item $d$, the depth of the goal state, 
\item $b$, the average branching factor of the tree, i.e.~the number of nodes generated per node expanded
\item The efficiency of the search.
\end{enumerate}
A perfect search algorithm, which heads directly to the goal, will nevertheless contain $O(bd)$ nodes as the alternative nodes must still be generated, even if they are never followed. An uninformed breadth-first search, on the other hand, will generate $O(b^d)$ nodes. This can be regarded as a sensible upper bound on the efficiency of the search (although it is possible to do worse).

Macro-operators tend to decrease $d$ at the expense of increasing $b$, so do very well in uninformed search when $d$ dominates, but show less advantage when a good heuristic exists, where $b$ and $d$ are equally important. So it becomes important to consider how to keep the increases in branching factor to a minimum. In the case of subgraph planning, there are two main reasons why $b$ increases:
\begin{enumerate}
\item The reduced graph may have a larger average degree than the original. Since a subgraph contains many vertices, it tends to have more out-going edges than a single vertex. If all these edges connect to different subgraphs, then the branching factor will be significantly larger.  Sparse subgraphs (such as halls) are worse in this regard than dense subgraphs (such as cliques). The subgraph decomposition needs to be chosen carefully to avoid this problem.

\item A single subgraph transition may create a large number of possible configurations, such as when a robot enters a large hall which is already occupied by several robots. In some cases it may not strictly matter which configuration is generated and where possible we use least commitment to avoid creating unnecessary alternatives, but if there is the possibility that different configurations will result in different outcomes further down the track, then they all need to be considered. Halls in particular have this problem.
\end{enumerate}
 
As we will see in the experiments that follow, careful choice of the subgraph decomposition is important to avoid these pitfalls, but with an appropriate partition the abstraction can significantly improve both informed and uninformed search.
 
\section{Experiments}
\label{sec:experiments}

To empirically test the advantages of the subgraph approach, we ran several experiments on both real and randomly generated problems. Our first experiment demonstrates how the algorithms scale with changes to the size of the problem, in terms of the number of vertices, edges and robots, under a standard breadth-first search. The second experiment shows how these results are affected by using an heuristic to guide search. Both of these experiments use randomly generated graphs. The final experiment demonstrates the algorithm on a realistic problem.

In the first two experiments, maps were generated randomly and automatically partitioned into subgraphs. Random generation was done as follows: first a spanning tree was generated by adding vertices one by one, connecting each to a randomly selected vertex in the graph. If further edges were required they were generated by randomly selecting two non-adjacent vertices and creating an edge between them. All edges were undirected.\footnote{It should be noted that this algorithm does not generate a uniform distribution over all connected graphs of a given size, but it is very difficult to generate sparse connected graphs with a uniform distribution. The bias is not deemed significant.}

Automated partitioning worked as follows:
\begin{enumerate}
\item Initially mark all vertices as unused.
\item Select a pair of adjacent unused vertices.
\item Use this pair as the basis for growing a hall, a ring and a clique:
\begin{description}
\item{\textbf{Hall:}} Randomly add unused vertices adjacent to either end of the hall, provided they do not violate the hall property. Continue until no further growth is possible.
\item{\textbf{Ring:}} Randomly add unused vertices adjacent to either end of the ring until a loop is created. Discard any vertices not involved in the loop. 
\item{\textbf{Clique:}} Randomly add unused vertices adjacent to every vertex in the clique. Continue until no further growth is possible. 
\end{description}
\item Keep the biggest of the three generated subgraphs. Mark all its vertices as used.
\item Go back to step 2, until no adjacent unused pairs can be found.
\item All remaining unused vertices are singletons.
\end{enumerate}
This is not intended to be an ideal algorithm. Its results are far from optimal but it is fast and effective. Experience suggests that a partition generated by this approach can contain about twice as many subgraphs as one crafted by hand, and it makes no effort to minimise the degree of the reduced graph, but even with these randomly generated partitions the advantages of the subgraph abstraction are apparent. 

\begin{figure}[h]
\begin{center}
\subfigure{\includegraphics[scale=0.6]{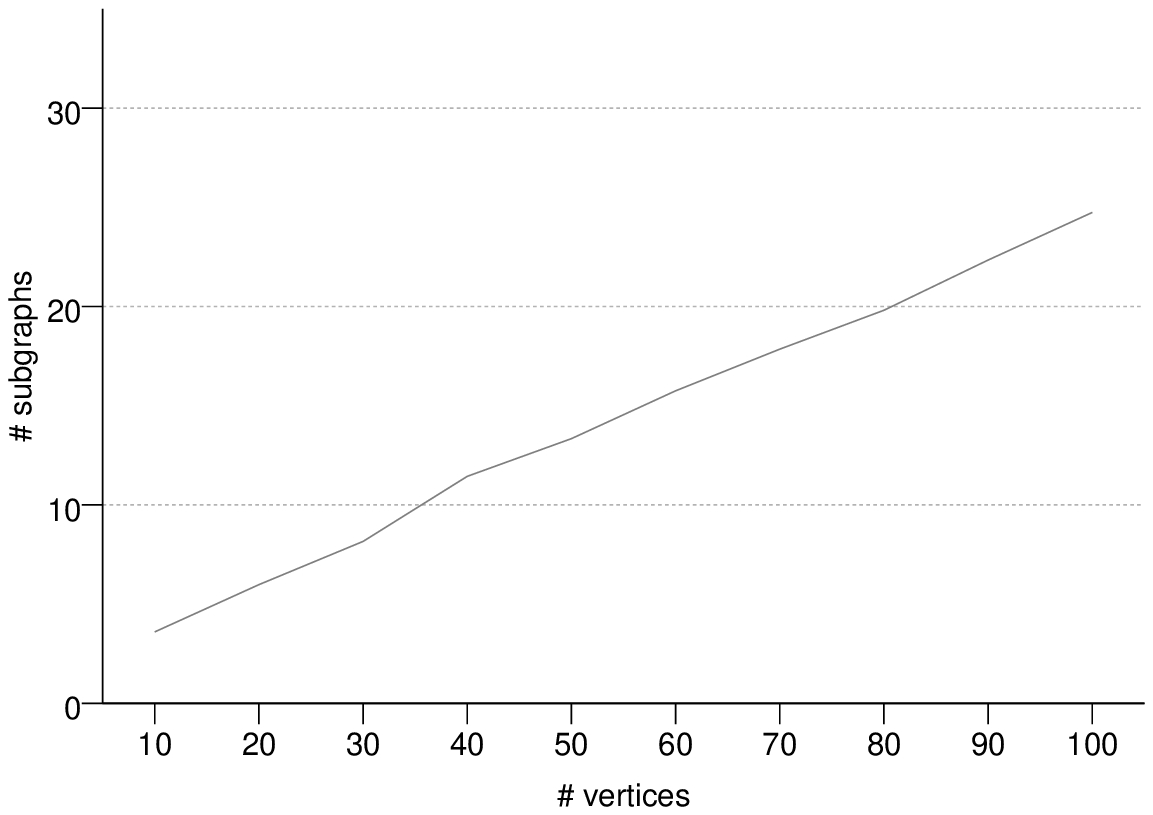}}
\subfigure{\includegraphics[scale=0.6]{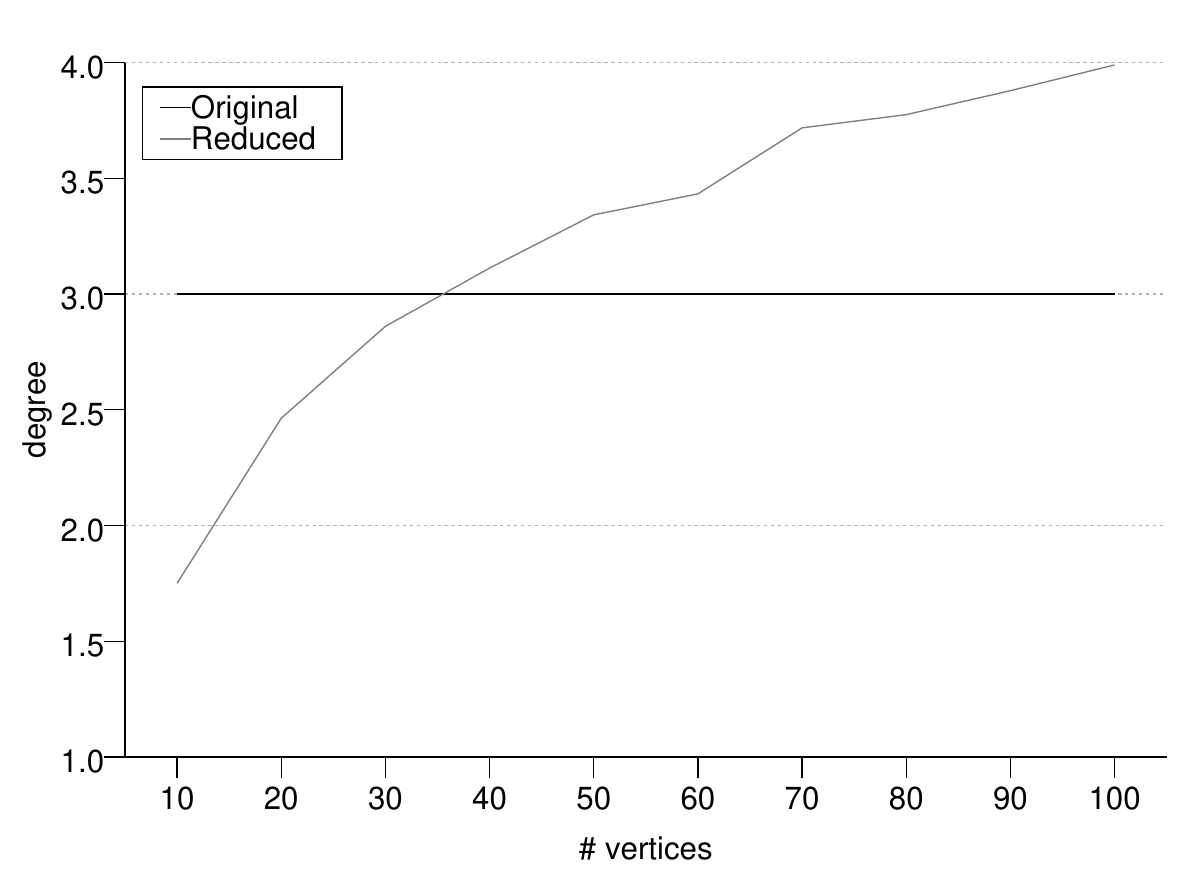}}
\end{center}
\caption{The results of the automatic partitioning program in Experiment 1a. The left graph shows the average number of subgraphs generated and the right graph shows the average degree of the reduced graph.} 
\label{fig:exp1subgraphs}
\end{figure}

\subsection{Experiment 1: Scaling Problem Size}
\label{sec:exp1}

\subsubsection{Scaling \ensuremath{|V|}}
\label{sec:exp1a}
In the first experiment we investigate the effect that scaling the number of vertices in the graph has on search time. Random graphs were generated with the number of vertices ranging from 10 to 100. Edges we added so that the average degree $d = |E| / |V|$ was always equal to 3. (This value seems typical for the realistic maps.) One hundred graphs were generated of each size, and each one was partitioned using the method described above. 

Figure~\ref{fig:exp1subgraphs} shows the performance of the auto-partitioning. As we can see, the number of subgraphs increased roughly linearly with the size of the graph, with an average subgraph size of 4. For small graphs (with fewer than 40 vertices) the reduced graph after partitioning is sparser than the original, but as the size increases the average degree of the reduced graph gets larger. These results are presented for informative purposes only. We make no claims about the quality of this partitioning algorithm, other than that it is indeed reducing the size of the graph, if only by a small factor.

\begin{figure}[]
\begin{center}
\subfigure[run times]{\includegraphics[]{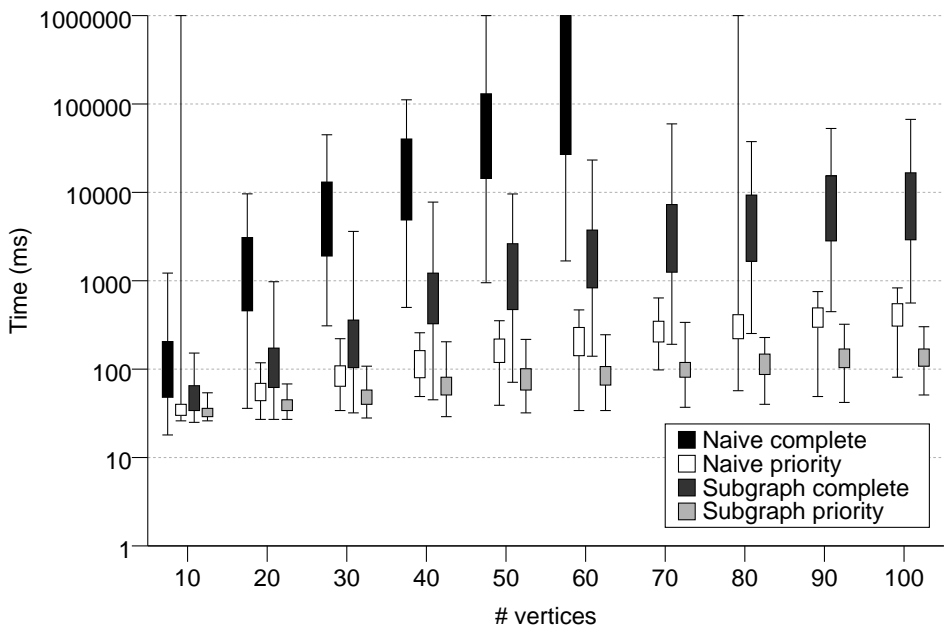}}
\subfigure[branching factor]{\includegraphics[scale=0.6]{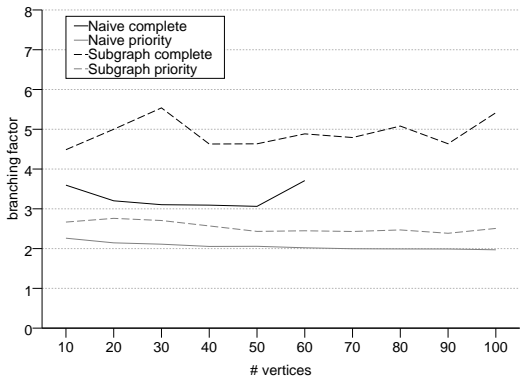}}
\subfigure[goal depth]{\includegraphics[scale=0.6]{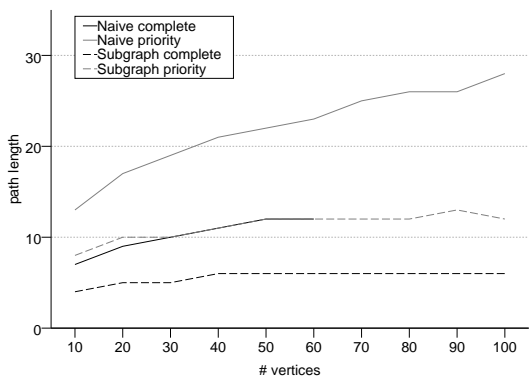}}
\end{center}
\caption{The results of Experiment 1a.  In graph (a) the boxes show the first and third quartile and whiskers to show the complete range. When an experiment failed to complete due to time or memory limits or incompleteness of the search, the run time was treated as infinite. No value is plotted for cases where more than 50\% of experiments failed. In graph (c) the goal depth for the naive complete and subgraph priority approaches are identical for graphs of 30 to 60 vertices, so the lines overlap. The naive complete planner could not solve problems with more than 60 vertices.} 
\label{fig:exp1}
\end{figure}
 
In each graph, three robots were given randomly selected initial and final locations, and a plan was generated. Figure~\ref{fig:exp1}(a) shows the average run times for each of the four approaches.\footnote{It has been noted that these times are overall rather slow. We acknowledge this and attribute it to our implementation, which is in Java and which was not heavily optimised to avoid garbage collection. We are currently working on an implementation with an optimised search engine, but we believe that these results still provide a valuable comparison between methods.} It shows a clear performance hierarchy. The complete planners are significantly slower than the priority planners, and in both cases the subgraph abstraction shows a significant improvement over the naive alternative. Nevertheless, in every case the combinatorial growth in runtime is apparent (note that the graph is plotted on a log scale). The linear relationship between number of vertices and number of subgraphs prevents the subgraph approaches from doing better than this. A better partitioning algorithm should ameliorate this problem.

To analyse the causes of this variation in run times, we need to consider the search process more carefully. We can measure the search depth $d$ and average branching factor $b$ for each experiment. The results are plotted in Figure~\ref{fig:exp1}(b) and (c). As we expected, when the subgraph abstraction is used, the goal depth is decreased and grows more slowly, but the branching factor is increased. Since we are doing uninformed search, $d$ dominates and the overall result is an improvement in planning time.

\begin{table}
\begin{center}
\caption{The number of planning failures recorded by the two prioritised planning approaches in Experiment 1a.}
\label{tab:exp1failures}
\begin{tabular}{|c|c|c|}
\hline
& \multicolumn{2}{c|}{\# Failures}  \\
\cline{2-3}
Vertices & Naive & Subgraph \\
\hline
10 & 2 & 0 \\
20 - 70 & 0 & 0 \\
80 & 1 & 0 \\
90 - 100 & 0 & 0 \\
\hline 
\end{tabular}
\end{center}
\end{table}
The incompleteness of prioritised planning shows in Table~\ref{tab:exp1failures}. On three occasions the naive prioritised search failed to find available solutions. However this was not a problem for the prioritised subgraph search.
 
\subsubsection{Scaling \ensuremath{|E|}}
\label{sec:exp1b}
Next we examine the effect of graph density. Fixing the number of vertices at 30, we generated random graphs with average degree ranging from $2.0$ to $4.0$. For each value, 100 graphs were randomly generated and automatically partitioned. Again the planning problem was to move three robots from between selected initial and goal locations. 

\begin{figure}[]
\begin{center}
\subfigure[run times]{\includegraphics[]{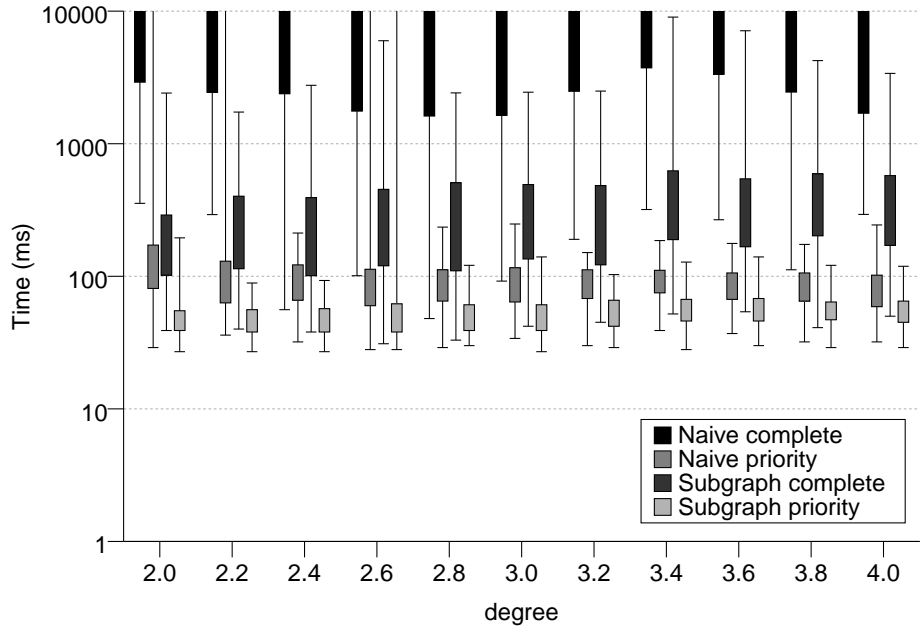}}
\subfigure[branching factor]{\includegraphics[scale=0.6]{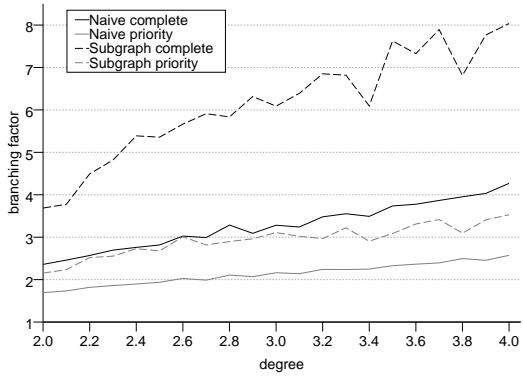}}
\subfigure[goal depth]{\includegraphics[scale=0.6]{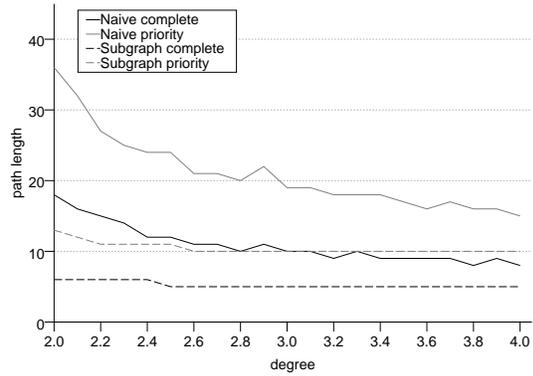}}
\end{center}
\caption{The results for Experiment 1b.} 
\label{fig:exp2}
\end{figure}

The results for this experiment are shown in Figure~\ref{fig:exp2}. There does not appear to be much overall change in the run times of any of the approaches, other than a small improvement from the naive prioritised planner as the graph gets denser. Figures~\ref{fig:exp2}(b) and (c) show the expected result: increasing the density of the graph increases the branching factor but decreases the depth. It appears to affect all four approaches similarly.

\begin{table}
\begin{center}
\caption{The number of planning failures recorded by the two prioritised planning approaches in Experiment 1b.}
\label{tab:exp2failures}
\begin{tabular}{|c|c|c|}
\hline
& \multicolumn{2}{c|}{\# Failures}  \\
\cline{2-3}
Degree & Naive & Subgraph \\
\hline
2.0 & 10 & 0 \\
2.2 & 8 & 0 \\
2.4 & 5 & 0 \\
2.6 & 1 & 1 \\
2.8 & 0 & 0 \\
3.0 & 2 & 0 \\
3.2 & 1 & 1 \\
3.4 - 4.0 & 0 & 0 \\
\hline 
\end{tabular}
\end{center}
\end{table}

An interesting difference, however, is shown in Table~\ref{tab:exp2failures}. This records the percentage of experiments for which each of the prioritised planners was unable to find a solution. For very sparse graphs, the naive planner failed on as many as 10\% of problems, but it improved quickly as density increased. With the subgraph abstraction added, the planner was able to solve all but two of the problems. In no case did we find problems which were solved by the naive planner and not by the subgraph planner. 

\subsubsection{Scaling \ensuremath{|R|}}
\label{sec:exp1c}

\begin{figure}
\begin{center}
\subfigure[run times]{\includegraphics[]{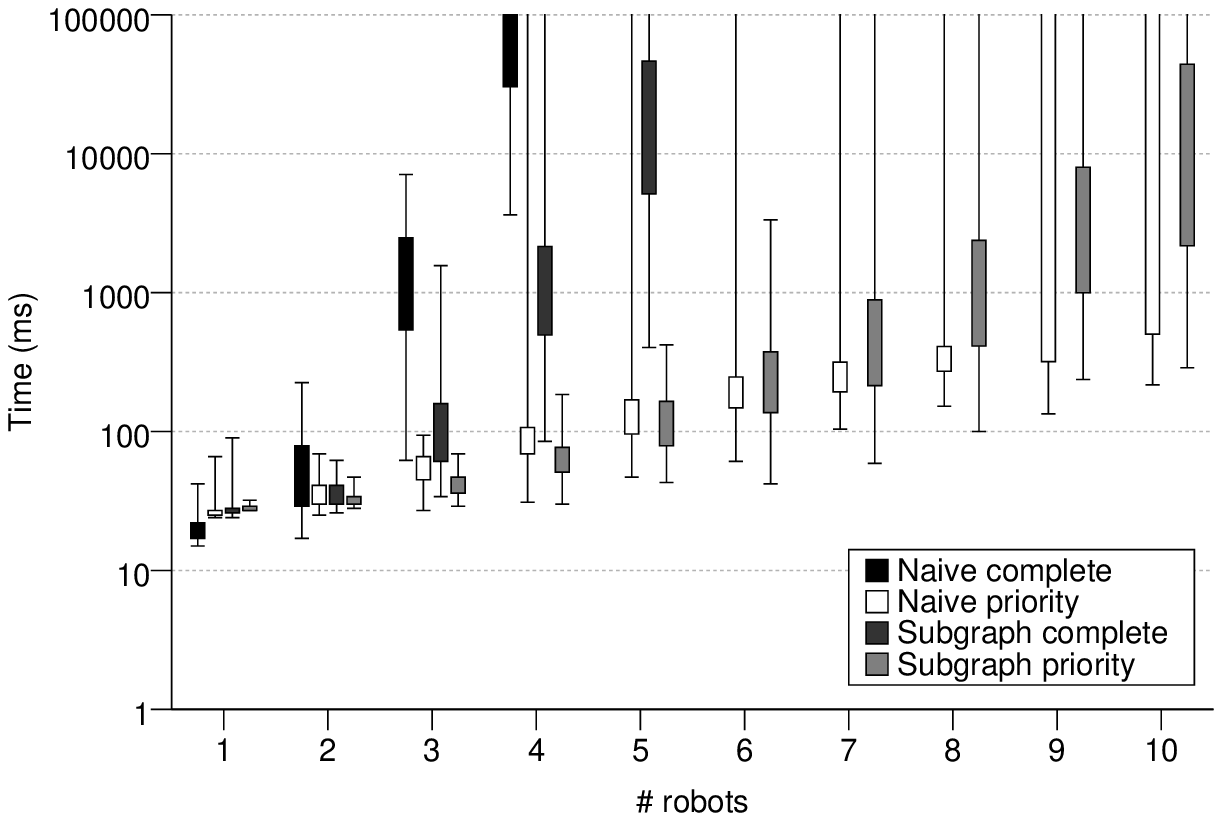}}
\subfigure[branching factor]{\includegraphics[scale=0.6]{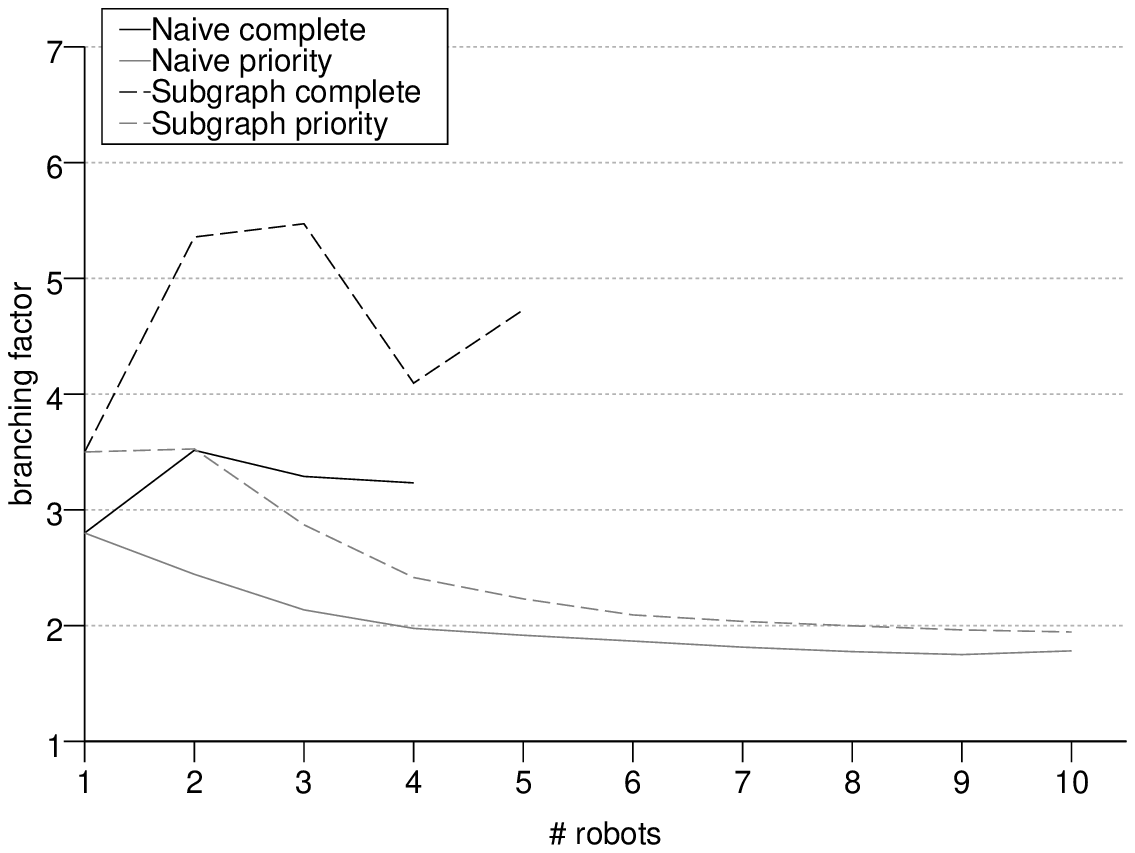}}
\subfigure[goal depth]{\includegraphics[scale=0.6]{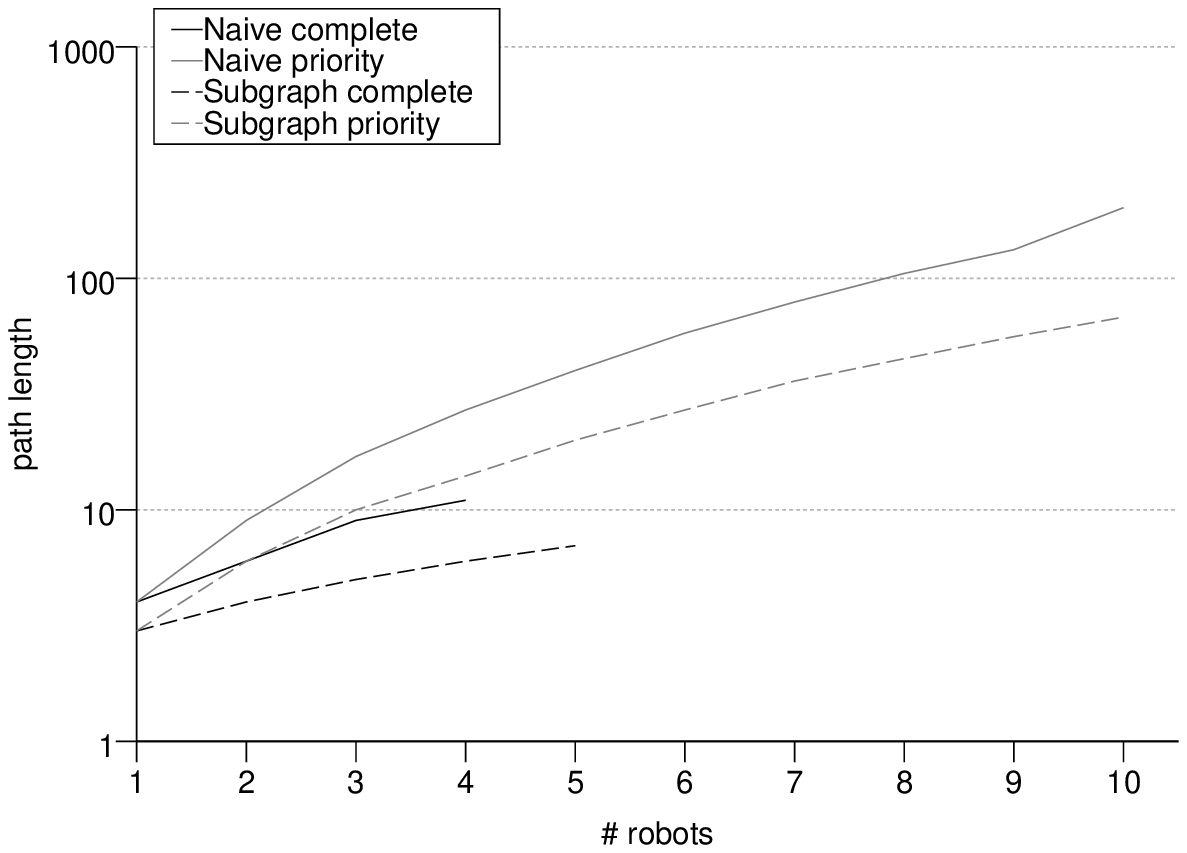}}
\end{center}
\caption{The results for Experiment 1c.} 
\label{fig:exp3}
\end{figure}


In the last of the scaling experiments, we investigate how each approach performs with varying numbers of robots. As before, 100 random graphs were generated and partitioned, each with 30 vertices and average degree of 3, and each one was partitioned using the automatic partitioning algorithm. Ten planning problems were set in each graph with the number of robots varying from 1 to 10. In each case initial and goal locations were selected randomly.

The running times for all four approaches are plotted in Figure~\ref{fig:exp3}(a). There is a major performance difference between the prioritised and non-prioritised planners, with the prioritised planners able to handle twice as many robots. Between the two complete-search approaches, the subgraph abstraction is an unnecessary overhead for very small problems, but shows significant advantage as the number of robots increases.

\begin{table}
\begin{center}
\caption{The number of planning failures recorded by the two prioritised planning approaches in Experiment 1c.}
\label{tab:exp3failures}
\begin{tabular}{|c|c|c|}
\hline
& \multicolumn{2}{c|}{\# Failures}  \\
\cline{2-3}
\# Robots & Naive & Subgraph \\
\hline
1 - 3 & 0 & 0 \\
4 & 3 & 0 \\
5 & 4 & 0 \\
6 & 10 & 0 \\
7 & 7 & 1 \\
8 & 7 & 1 \\
9 & 26 & 0 \\
10 & 46 & 1 \\
\hline 
\end{tabular}
\end{center}
\end{table}

There is less obvious advantage to the subgraph abstraction in the case of prioritised planning, until we look at the failure rates shown in Table~\ref{tab:exp3failures}. As the number of robots increases the incompleteness of the naive prioritised algorithm begins to become apparent, until with 10 robots we see that 46\% of the problems could not be solved by this planner. The advantage of the subgraph abstraction is now apparent: only a total of 3 problems could not be solved out of 1000 tried. 

Figures~\ref{fig:exp3}(b) and (c) plot the average branching factor and goal depth for these problems. As in previous experiments, the subgraph abstraction is seen to increase the branching factor but decrease the depth. In the complete search approaches the branching factor grows rapidly with the number of robots, as each node on the search path contains a choice of which robot to move. The prioritised approach reverses this trend, as planning is only ever done for one robot at a time, and the later robots are much more heavily constrained in the options available to them, providing fewer alternatives in the search tree.

\subsubsection{Discussion}

To summarise the above experiments, the advantages of the subgraph abstraction are two-fold. Firstly, it decreases the necessary search depth of a planning problem by compressing many robot movements into a single abstract step. Like other macro-based abstractions, it does this at the expense of increasing the branching factor but the gains seem to outweigh the losses in practice. Of course, this is dependent to some degree on the use of uninformed search, which we shall address below.

The other advantage is specific to the prioritised planner. For tightly constrained problems with sparse maps and/or many robots the incompleteness of the naive prioritised search becomes a very significant issue. With the addition of the subgraph abstraction the number of such failures is dramatically reduced, without additional search.

\subsection{Experiment 2: Heuristic Search}
\label{sec:exp2}

All the experiments so far have involved uninformed breadth-first search without the use of an heuristic. As such, the runtime of the algorithms is more strongly affected by changes in search depth than by the branching factor. As we explained above, uninformed search has an $O(b^d)$ expected running time. However a perfect heuristic can reduce this to $O(bd)$, making the branching factor a much more significant aspect. A perfect heuristic is, of course, unavailable, but it it possible to efficiently compute a reasonably good search heuristic for this task by relaxing the problem. Disregarding collisions we can simply compute the sum of the shortest path lengths from each robot's location to its goal. This is an underestimate of the actual path length, but is accurate for loosely constrained problems (with few robots and dense graphs).

\begin{figure}[b!]
\begin{center}
\subfigure[subgraphs]{\includegraphics[scale=0.6]{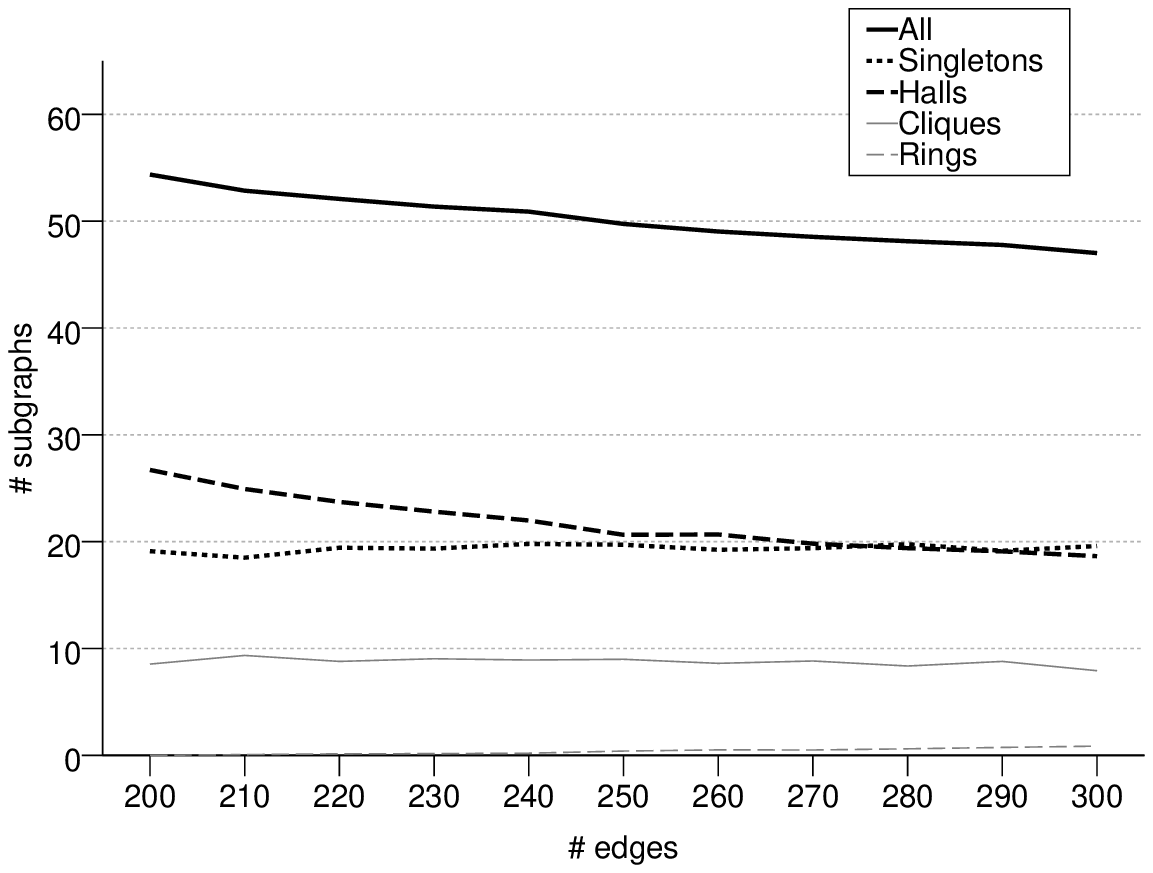}}
\subfigure[degree]{\includegraphics[scale=0.6]{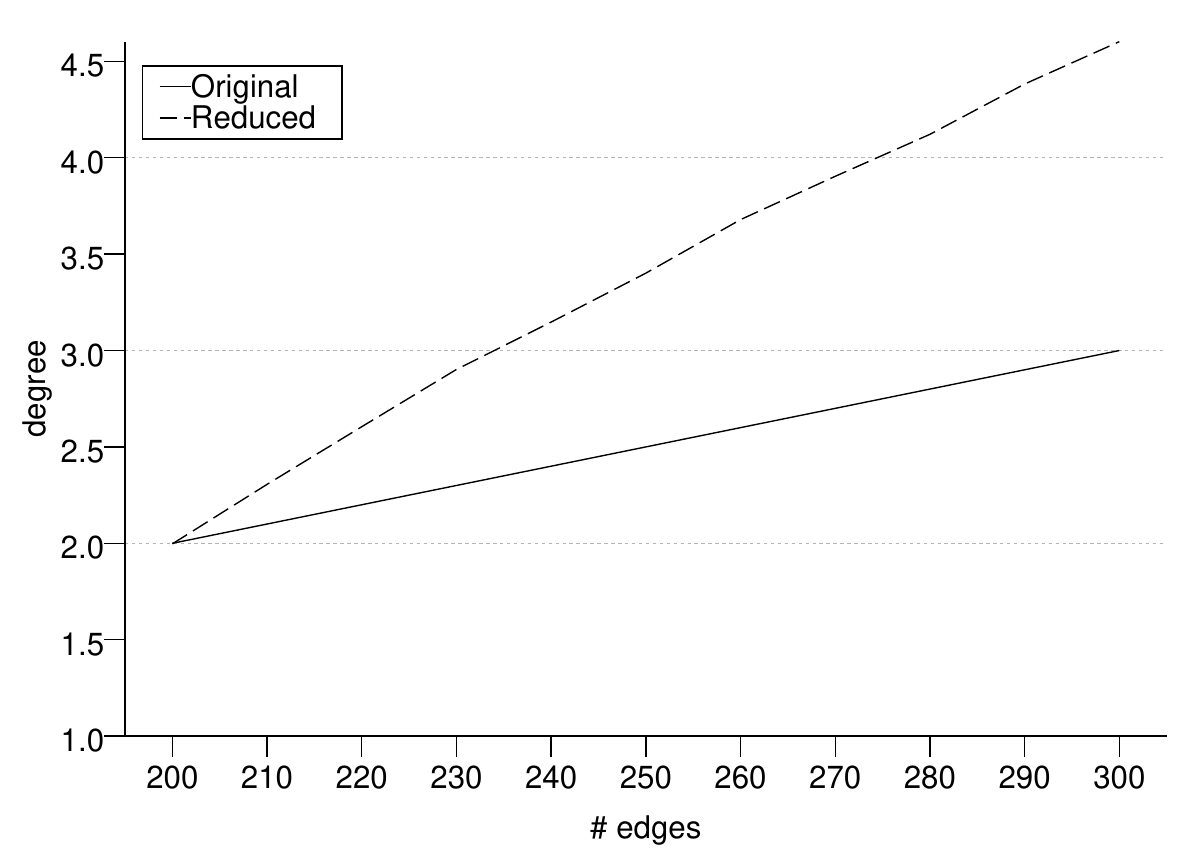}}
\end{center}
\caption{The results of the auto-partitioner on graphs in Experiment 2.} 
\label{fig:exp4partition}
\end{figure}

\begin{figure}
\begin{center}
\subfigure[run times]{\includegraphics[]{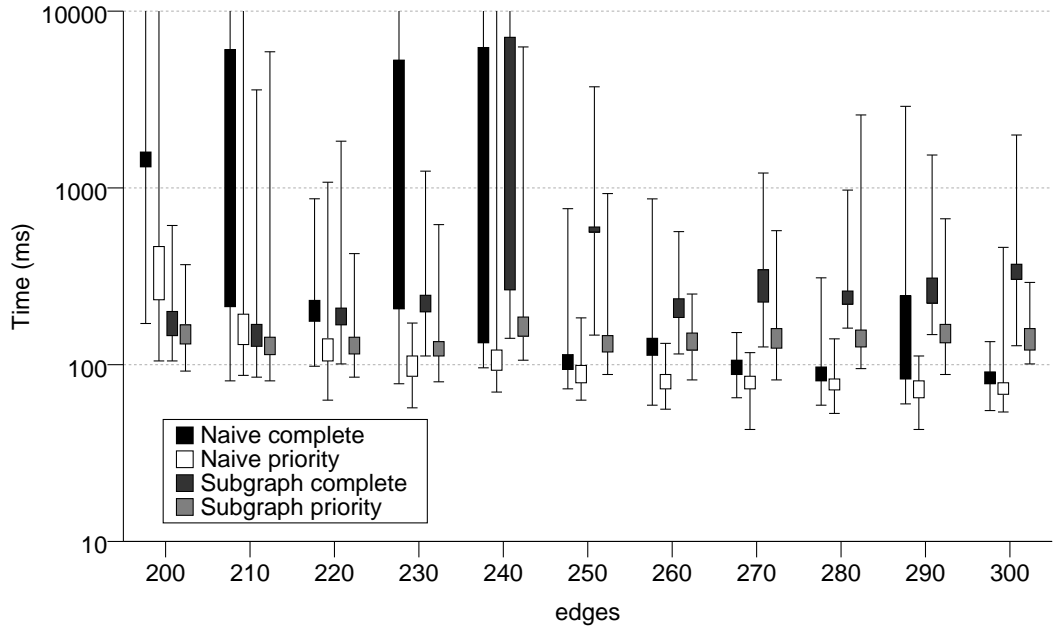}}
\subfigure[branching factor]{\includegraphics[scale=0.6]{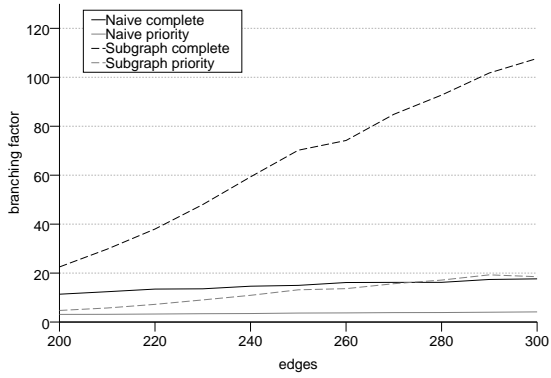}}
\subfigure[goal depth]{\includegraphics[scale=0.6]{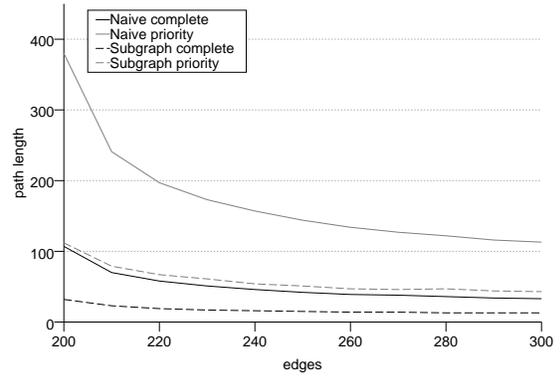}}
\end{center}
\caption{The results for Experiment 2.} 
\label{fig:exp4results}
\end{figure}

In this experiment we used a best-first search algorithm guided by this heuristic.\footnote{The A* algorithm was not used, as we have no desire to minimise the length of the solution, just to find a solution as quickly as possible.} At every node in the search tree, we selected the plan which minimised this value. In the case of the subgraph planner, the actual locations of robots at any time-point are not specified, just the subgraph they occupy, so the heuristic was calculated using the maximum distances from any vertex in each robot's subgraph to its goal. We pre-computed the shortest path distances between every pair of nodes before running the planner, so the time to do this computation is not counted in the runtime for the algorithm.

The utility of this heuristic depends largely on how constrained the problem is. If the graph is dense and there are relatively few robots, the heuristic should direct the planner quickly to the goal. However if the graph is sparser, then interactions between robots will become more important, and the heuristic will be less useful. For this reason, we concentrate our attention in this experiment on how varying the density of the graph affects the performance of our different approaches.

Random maps of 200 vertices were generated, with average degree ranging from 2 to 3. One hundred graphs were generated of each size and partitioned using the algorithm described earlier. Figure~\ref{fig:exp4partition} shows the results. As the original graph gets denser, the number of subgraphs decreases, mostly because it is possible to create longer halls. This is good, as fewer subgraphs mean shorter paths, but the consequential increase in degree will adversely affect the branching factor.

Ten robots were placed randomly in each graph and assigned random goal locations. All four planning approaches were applied to these problems. The resulting run-times are plotted in Figure~\ref{fig:exp4results}(a). The first thing that is apparent from this graph is that the distinction between the different approaches is greatly reduced. Both the size of the graph and the number of robots are much larger than in previous experiments, and this has had a corresponding effect on the goal depth and branching factor (Figure~\ref{fig:exp4results}(b) and (c)), but the run-times are much smaller, so clearly the heuristic is effective at guiding the search. On average the ratio of search nodes expanded to goal depth was very close to 1.0 in all experiments, with only a slight increase in the more constrained cases, so we can conclude that this heuristic is close to perfect. 

\begin{table}
\begin{center}
\caption{The number of planning failures recorded by the two prioritised planning approaches in Experiment 2.}
\label{tab:exp4failures}
\begin{tabular}{|c|c|c|}
\hline
& \multicolumn{2}{c|}{\# Failures}  \\
\cline{2-3}
\# Edges & Naive & Subgraph \\
\hline
200 & 14 & 0 \\
210 & 2 & 0 \\
220 & 0 & 0 \\
230 & 0 & 0 \\
240 & 1 & 0 \\
250 - 300 & 0 & 0 \\
\hline 
\end{tabular}
\end{center}
\end{table}


When we compare the four approaches we see three distinct stages. In the most constrained case, at 200 edges, we see both the subgraph approaches outperforming either naive approach, with a small benefit in prioritised search over complete search. At 220 edges the pattern has changed. The two prioritised methods are significantly better  than the two complete approaches. As the number of edges increases, both the naive methods continue to improve, while prioritised subgraph search holds steady and complete subgraph search gets significantly worse (due to its rapid increase in branching factor). At 300 edges both the naive approaches are doing significantly better than the subgraph approaches.

The cause is clearly seen in Figures~\ref{fig:exp4results}(b) and (c). The branching factors for the subgraph approaches increase significantly faster than for the naive approaches, and the corresponding improvement in goal depth is not sufficient to outweigh the cost.

The benefits of the subgraph abstraction in highly constrained cases is also shown in the failure cases (Table~\ref{tab:exp4failures}). At 200 edges the naive prioritised search was unable to solve 10\% of problems, while prioritised search with subgraphs could solve them all. The number of failures fell quickly as the density of the graph increased.

\subsubsection{Discussion}

Once a graph becomes moderately dense and interactions between robots become few, the total-single-robot-paths measure becomes a near perfect heuristic. This makes the branching factor a much more critical factor than when using uninformed search. The auto-partitioning algorithm we use does a very poor job limiting this factor and so the subgraph approaches perform poorly. 

Better results could be achieved with better decomposition, but it is not clear whether this could be found in a random graph without excessive computation. Certainly partitioning such graphs by hand is no easy task. Realistic graphs, on the other hand, are generally shaped by natural constraints (e.g. rooms, doors and corridors) which make decomposition much simpler, as we will see in the following experiment.

\subsection{Experiment 3: The Indoor Map}
\label{sec:exp3}

Figure~\ref{fig:k17map} shows the map for our final two experiments, based on the 
floor-plan of Level 4 of the K17 building at the University of New South Wales. A road-map 
of 113 vertices and 308 edges has been drawn (by hand) connecting all the offices and open-plan desk locations. It is imagined that this might be used as a map for a delivery task involving a team of medium-sized robots.

\begin{figure}
\begin{center}
\includegraphics[width=\textwidth]{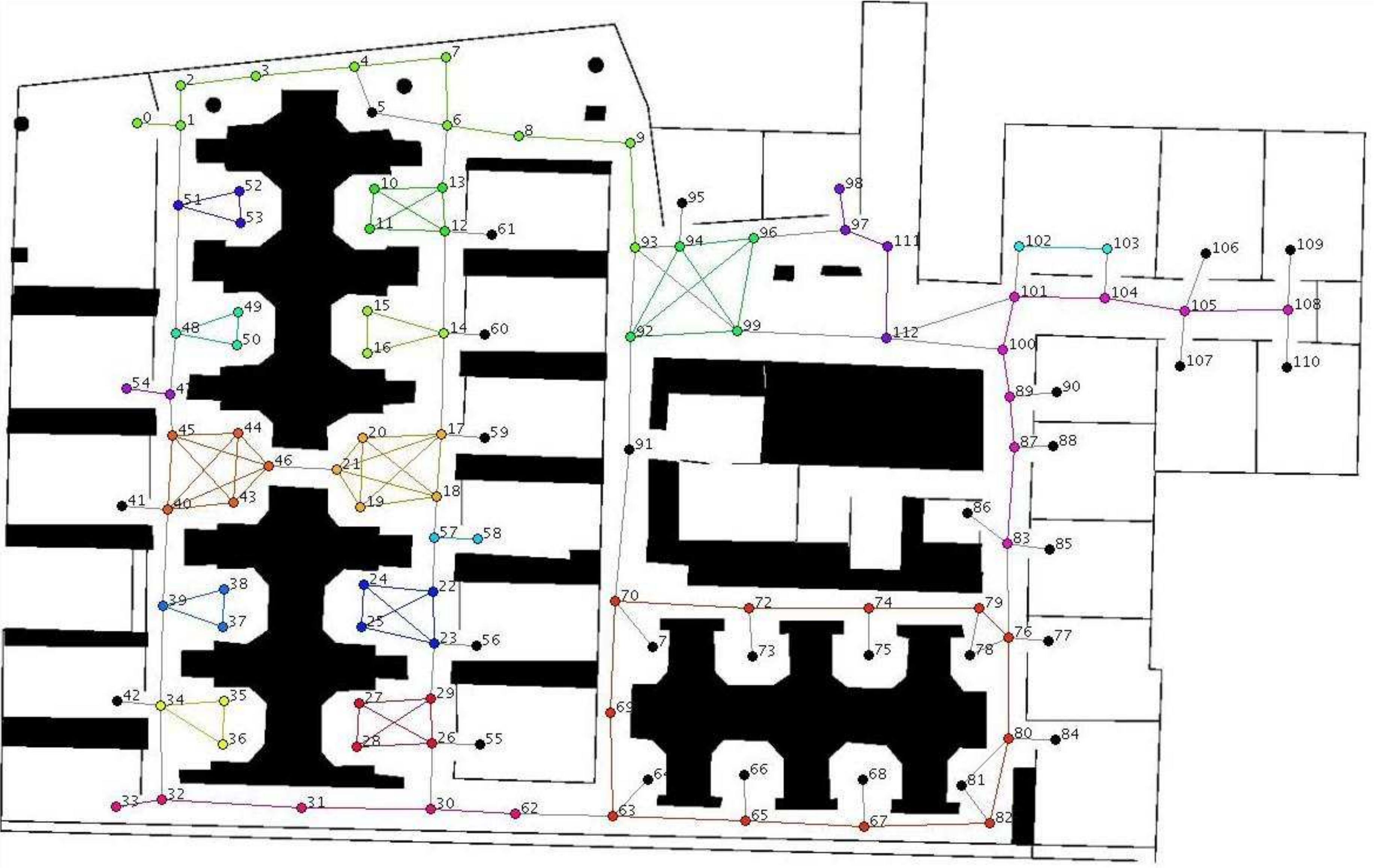}
\end{center}
\caption{The map for Experiment 3. Vertices are coloured by subgraph.} 
\label{fig:k17map}
\end{figure}

The road-map has been partitioned into 47 subgraphs -- 11 cliques, 7 halls and 1
ring, plus 28 remaining 'singleton' nodes (subgraphs containing only one
vertex). The average degree of the reduced graph is 2.1, compared to 2.7 in the original.\footnote{In comparison, the auto-partitioner yielded a partition with fewer subgraphs (avg. 41.8) but higher degree (avg. 2.25).} Partitioning was done by hand with the aid of an interactive GUI
which performed some simple graph analysis and offered recommendations (by indicating nodes which could be added to a hall or clique the user is creating). The
road-map was clearly laid out with partitioning in mind and deciding on this
partitioning was not on the whole difficult. Large open spaces generally became
cliques. Corridors became halls or rings. Only the foyer area (around vertex 94)
caused any particular trouble when finding an ideal partitioning, due to its
slightly unusual topology.\footnote{For the curious, the empty rooms in the
centre of the map, near vertex 91, are bathrooms. We did not consider that the
robots would need to make deliveries there.}

\begin{figure}
\begin{center}
\includegraphics[width=\textwidth]{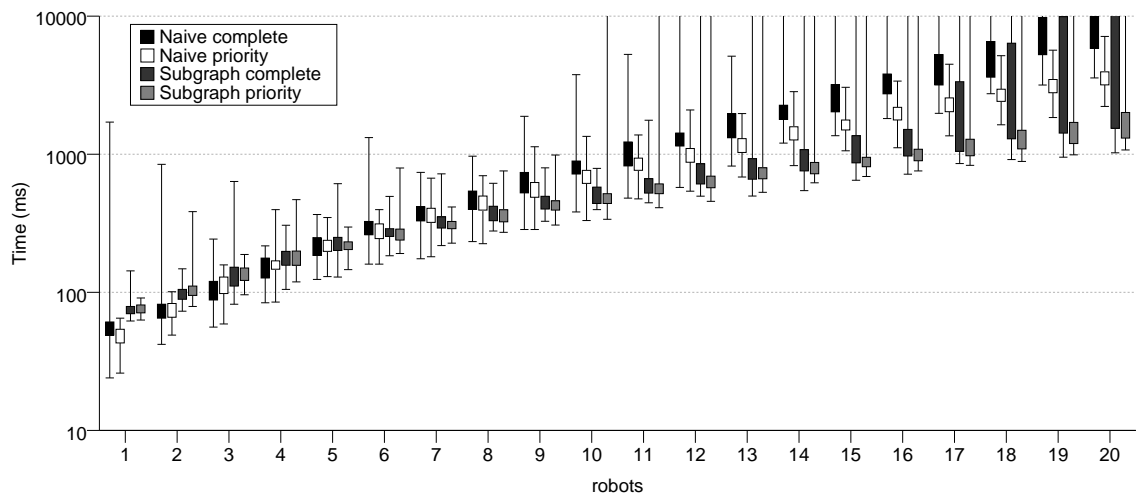}
\end{center}
\caption{Comparing run times for Experiment 3.}
\label{fig:K17runtime}
\end{figure}

A series of experiments were run in this world, varying the number of robots 
from 1 to 20. For each experiment 100 runs were performed in which each robot 
was placed in a random office or desk and was required to make a delivery 
to another random office or desk (chosen without replacement, so no two robots had 
the same goal). Plans were built using both complete and prioritised planners with and without the subgraph abstraction. All four approaches utilised the total single-robot shortest path heuristic from the previous experiment. The running times of each algorithm are shown in Figure~\ref{fig:K17runtime}.

We can see that for small numbers of robots (1 or 2) the naive approaches are significantly better than the subgraph approaches. The overhead of doing subgraph search outweighs its disadvantages in such simple problems. As the number of robots increases the subgraph methods take over, and for around 9 to 16 robots both subgraph methods are significant better than either naive approach.  At 17 robots the combination of complete search with subgraphs begins to perform less well and the two prioritised approaches are the best performers, with a considerable advantage to the subgraph approach.

\begin{figure}[bt]
\begin{center}
\includegraphics[]{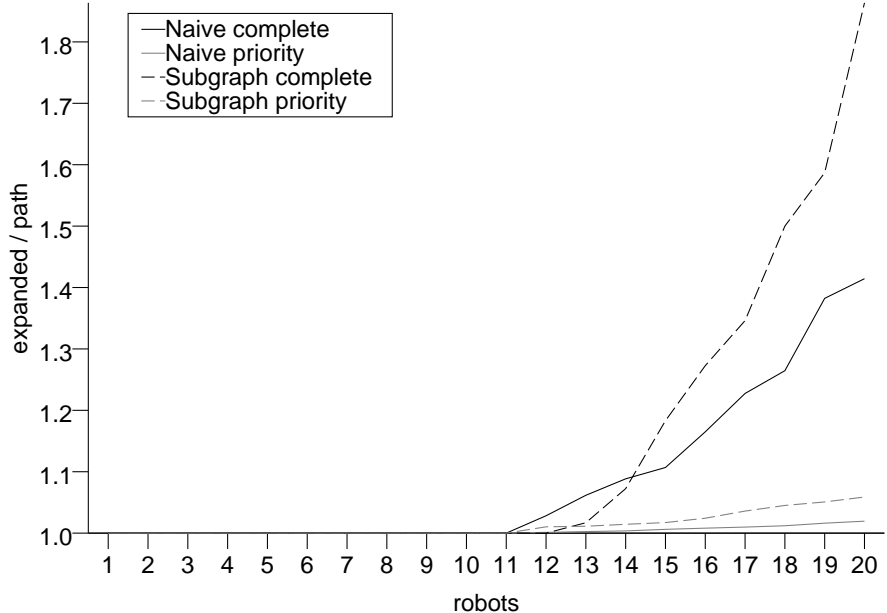}
\end{center}
\caption{Assessing the quality of the heuristic in Experiment 3. The value plotted is the ratio of the number of expanded nodes in the search tree and the goal depth. A perfect heuristic yields a value of 1.0.}
\label{fig:K17heuristic}
\end{figure}

Considering search complexity, let us first examine the performance of the heuristic. Figure~\ref{fig:K17heuristic} plots the ratio or the average number of expanded nodes in the search tree and the goal depth. For a perfect heuristic, this value is 1.0, as it is in this experiment for up to 11 robots. With more than 11 robots the heuristic begins to become inaccurate. The inaccuracy seems to affect the complete planners more badly than the prioritised ones, and in both cases the subgraph approach is more seriously affected than the naive approach. 

\begin{figure}
\begin{center}
\subfigure[branching factor]{\includegraphics[scale=0.6]{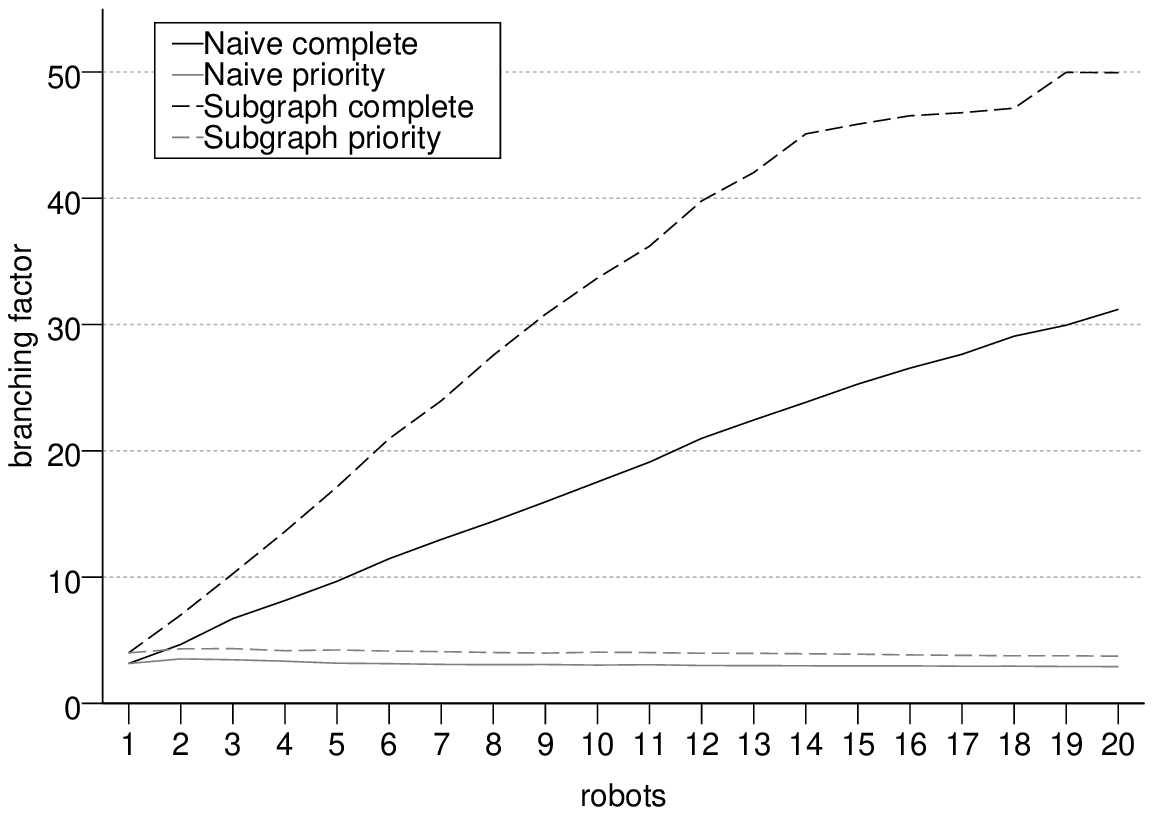}}
\subfigure[goal depth]{\includegraphics[scale=0.6]{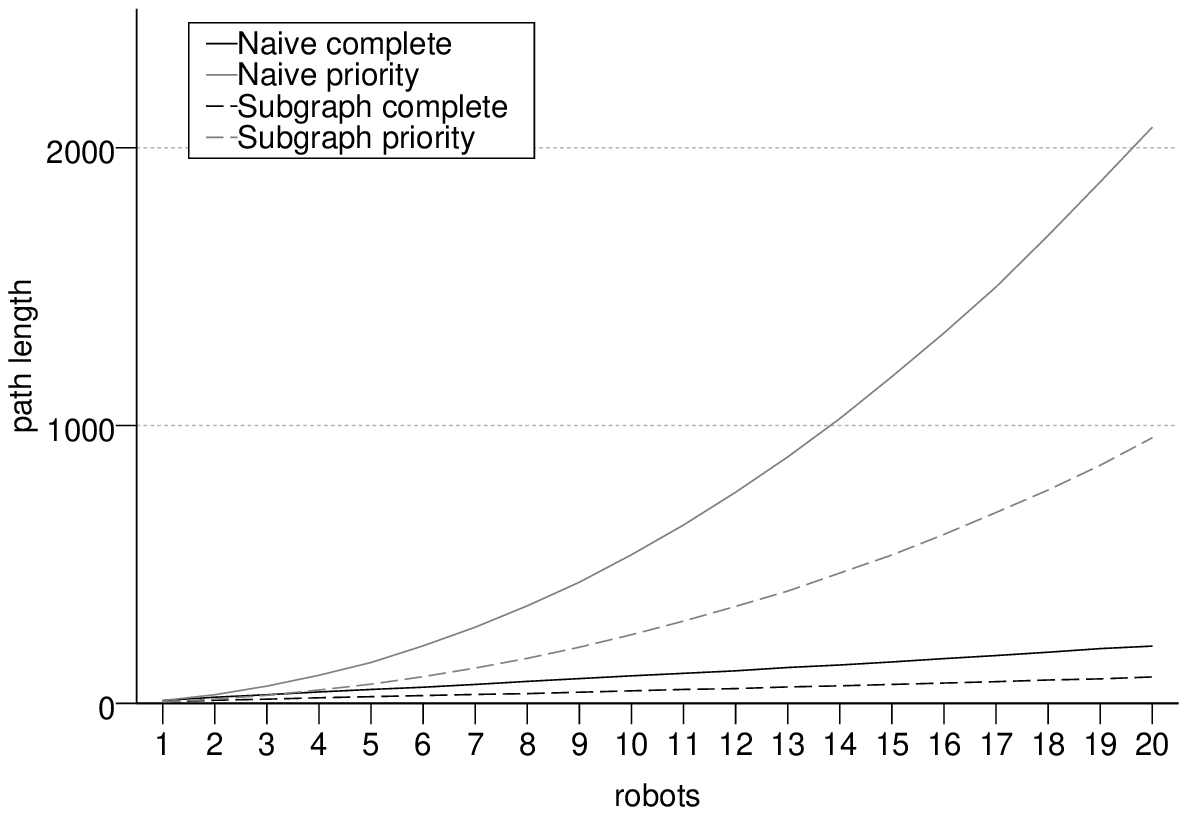}}
\end{center}
\caption{The branching factor and goal depth for Experiment 3.} 
\label{fig:K17bd}
\end{figure}

To explain this difference, note that the heuristic we are using contains significantly less information for subgraph search than it does for naive search. As we do not know exactly where a robot is within a subgraph, we assume that it is in the worst possible position. This means that the value of a configuration tuple is based solely on the allocation of robots to subgraphs, and not on the particular configurations of those subgraphs. Hall subgraphs in particular may have several different configurations for the same set of robots, which will all be assigned the same heuristic value despite having significantly different real distances to the goal.This creates a plateau in the heuristic function which broadens the search. For large numbers of robots these permutations become a significant factor in the search. To improve the heuristic we need to find a way to distinguish the value of different configurations of a subgraph. This will probably require an extra method for each specific subgraph structure.

The graphs of branching factor and goal depth (Figure~\ref{fig:K17bd}) show what we have come to expect~-- the branching factor is larger in the complete search than in prioritised search and the subgraph abstraction makes it worse. Significantly, the branching factor for prioritised search does not increase as more robots are added, because at any step in the plan only one robot can be moved. The goal depth shows the opposite pattern, complete searches are shorter than prioritised searches and the subgraph abstraction approximately halves the search depth in all cases.

\begin{table}
\begin{center}
\caption{The number of planning failures recorded by the two prioritised planning approaches in Experiment 3.}
\label{tab:K17failures}
\begin{tabular}{|c|c|c|}
\hline
& \multicolumn{2}{c|}{\# Failures}  \\
\cline{2-3}
Edges & Naive & Subgraph \\
\hline
1 - 9 & 0 & 0 \\
10 - 19 & 0 & 1 \\
20 & 0 & 2 \\
\hline 
\end{tabular}
\end{center}
\end{table}


Failure rates are recorded in Table~\ref{tab:K17failures}. The story here is different from that of previous experiments. The naive prioritised planner was able to solve all the problems at every depth, but adding the subgraph abstraction caused a small number of failures in more complex problems. It is not clear what has caused this reversal. The cases involved are very complex and elude analysis. This problem warrants further investigation. 

\subsubsection{Discussion}

This experiment has shown that in a realistic problem with an appropriately chosen set of subgraphs the subgraph abstraction is an effective way to reduce the search even when a good heuristic is available. Why does the subgraph abstraction work so well in this example, compared to the random graphs in Experiment 2? The answer seems to be found in the degree of the reduced graph. Automatically partitioning a random graph significantly increases its degree, as we saw in Figure~\ref{fig:exp4partition}(b). This, in turn, increases the branching factor and thus the search time. 

In contrast, when we partition the realistic map we decreased the degree of the graph from 2.7 to 2.1 (by hand) or 2.25 (automatically).  The branching factor for the subgraph methods is still larger (as one transition can still create multiple configurations) but the effect is reduced enough to be overcome by the decrease in goal depth. The indication is that a realistic map has more structure that can be exploited by this abstraction. More investigation is warranted to characterise the features that many this possible. 

\section{Conclusion}
\label{sec:conclusions}

We have demonstrated a new kind of abstract representation for multi-robot path
planning which allows for much faster planning without sacrificing completeness.
Decomposing a road-map into subgraphs is a simple and intuitive way of
providing background knowledge to a planner which can be efficiently exploited.
The key is to find subgraph structures which allow us to treat many arrangements
of robots as equivalent configurations and to compute transitions between these
configurations quickly and deterministically. We have described four such
structures in this paper: stacks, halls, cliques and rings. These structures are
simple enough to compute configurations easily but also common enough to be
found in many realistic maps.

We have shown that abstract plans on these subgraphs can be resolved
deterministically into concrete plans without the need for further search. The planner is sound and complete, although the plans produced are not necessarily optimal. Future
work could prove that it is worth spending more time in the resolution phase to
trim unnecessarily wasteful plans, using, for example, simulated annealing \cite{sanchez1999lop}. It may be that the time saved in abstract
planning leaves us space to do more clever resolution.

The conventional solution to the search-space explosion in multi-robot planning
is prioritisation. We have shown that not only is subgraph-based planning
competitive with prioritised planning but also that the combination of the two
methods is more powerful still and in some cases, partly alleviates the 
incompleteness of the prioritised approach.

\subsection{Related Work}
\label{sec:related}

Abstraction and hierarchical decomposition are standard techniques in planning 
and other related search problems. The use of macro-operators dates back as far as Sacerdoti's early work on the \textsc{Abstrips} planning system \cite{sacerdoti1974} which introduced abstraction hierarchies, whereby a problem could first be solved at a high level of abstraction while ignoring lower-level details. The idea has been re-expressed in many different ways through the history of planning -- far too many to review in detail. This present work was particularly inspired by the `generic types' of \citeA{longfox02} in which they similarly detected substructures in a task-planning problem and solved them using structure-specific planners. 

Hierarchical planning has been applied to path-planning before with abstractions 
such as approximate cell decomposition \cite{Barbehenn1995,Conte1995}, 
generalised Voronoi graphs \cite{choset1995sbp,Choset1996} and general 
ad-hoc hierarchical maps \cite{bakker2005hdp,zivkovic2005hmb,zivkovic2006hmb}, but the structures identified in these examples do not 
carry over well to the multi-robot scenario.

Other faster solutions to the multi-robot problem are available if we can assume the existence of ``garage'' locations for each robot \cite{lavalle98} or other kinds of temporary free space \cite{sharma1992cmp,fitch2003rph}. The method we present here makes no such assumption and is thus more general in application. 
There does not appear to be any previous work which provides a complete abstraction-based planner for the general multi-robot problem.

The work that bears most similarity to our own is not explicitly in robot path planning, but in solving the Sokoban puzzle \cite{botea2003,junghanns2001}. That domain is significantly more constrained than ours (the map is necessarily an orthogonal grid and the stones can only move when they are pushed by the man) but the method they employ is similar. Dividing a map up into \emph{rooms} and \emph{tunnels} they use the strongly-connected-component algorithm to identify equivalent arrangements of boulders in each subpart. Equivalent arrangements are then treated as a single abstract state -- corresponding to a configuration in our formulation -- which is used as the state in a global search. The particular structures they represent are different, but the general ideas of partitioning into independent local subproblems and identifying abstract states from strongly connected components, are the same as those employed in this work.

\subsection{Future Plans}

In the next stage of this project we plan to examine the symmetries provided by
the subgraph representation. Recent work in symbolic task-planning \cite{porteous04} has shown that recognising and exploiting symmetries and almost-symmetries in
planning problems can eliminate large amounts of search. Subgraph configurations
provide a natural ground for similar work in our problem domain and we expect
similar improvements are possible. 

We also plan to further investigate the problem of automatic subgraph partitioning of maps. Having identified the importance of trading off path depth against branching factor, we plan to make a partitioning algorithm which chooses subgraphs that optimise this relationship. Automatically finding an optimal partition could be very hard, but creating a powerful interactive partitioning tool for a human operator would seem to be a viable compromise. One approach would be to adapt the auto-partitioner we describe in this paper so that the seed vertices are selected by the user, who is then allowed to choose from a number of possible subgraphs based on this selection.

Further subgraph structures can also be identified, and we are currently working on formalising the properties of tree-structured subgraphs. Another possibility would be to generalise cliques and rings into a new `ring-with-chords' structure, although characterising such a structure may prove difficult. 

There have been many other advances in search technology which may be applicable to the multi-robot planning problem. We are currently in the process of re-expressing the entire problem as a constraint satisfaction problem (CSP) in the Gecode constraint engine \cite{gecode}.  We believe that the CSP formulation will be a powerful way to take advantage of the structural knowledge that subgraph decomposition represents.

\acks{I'd like to thank Jonathan Paxman, Brad Tonkes and Maurice Pagnucco for their help in developing the ideas in this paper and proofreading the drafts.}

\appendix

\section{Proof of Soundness and Completeness}

In this appendix we set up the necessary formal definitions and then prove the soundness and completeness of the abstract planning process. The main result is a theorem showing that an abstract plan exists for a given problem if and only if a concrete plan also exists. 

\subsection{Graphs and Subgraphs}
An \e{induced subgraph} $S \subseteq G$ is a graph $S = (V(S), E(S))$ such that
 \[ V(S) \subseteq V(G) \hspace{1cm} E(S) = \setof{(u,v)}{u,v \in V(S), (u,v)
 \in E(G)} \]
Intuitively this describes a subgraph consisting of a subset of vertices with
all their connecting edges from the parent graph. Thus an induced subgraph can
be specified solely in terms of its vertices. We shall henceforth assume that all
subgraphs we refer to are induced. 

A \e{partition} $\mathcal{P}$ of $G$ is a set  $\{S_1, \ldots, S_m\}$ of 
subgraphs of $G$ satisfying \[ V(G) = \bigcup_{i=1 \ldots m} V(S_i)
 \text{\hspace{1cm}and\hspace{1cm}}
V(S_i) \cap V(S_j) = \emptyset, ~\forall i,j : i \neq j \]
 
Given a graph $G$ and a partition $\mathcal{P}$ we can construct the \e{reduced
graph} $X$ of $G$ by contracting each subgraph to a single vertex
\begin{align*}
V(X) & = \mathcal{P} \\
E(X) & = \setof{(S_i, S_j)}{\exists x \in S_i, y \in S_j : (x,y) \in G}
\end{align*}
 
\subsection{Robots and Arrangements}
 
Let us assume we have a set of robots $R$. An \e{arrangement} $a$ of robots in 
a graph $G$ is a 1-to-1 partial function $a : V(G) \rightarrow R$. An 
arrangement represents the locations of robots within $G$. If $a(v) = r$, then 
robot $r$ is at vertex $v$. We shall use the notation $a(v) = \Box$ to indicate 
that $a$ is undefined at $v$, i.e.\ vertex $v$ is unoccupied. An 
arrangement may not necessarily include every robot in $R$. Two arrangements $a$
and $b$ are said to be \e{disjoint} if $range(a) \cap range(b) = \emptyset$. Let $A_G$
represent the set of all arrangements of $R$ in $G$. 

If $S$ is a subgraph of $G$, and $a$ is an arrangement of $R$ in $G$ then we 
define $a / S$, the \e{induced arrangement} of $R$ in $S$, as \[a/S(v) = 
a(v), ~\forall v \in V(S)\] 

If $S_1$ and $S_2$ are disjoint subgraphs of $G$ with disjoint arrangements 
$a_1$ in $S_1$ and $a_2$ in $S_2$, then we define the \e{combined arrangement} 
$a = a_1 \otimes a_2$ as an arrangement in $S_1 \cup S_2$ satisfying \[a(v) = 
\begin{cases} a_1(v) & \text{~if~} v \in S_1 \\
a_2(v) & \text{~if~} v \in S_2 \end{cases}\]

\begin{lemma} \label{lem:induced_arrangements}
If $a$ is an arrangement in $G$ with partition $\mathcal{P}
= \{S_1, \ldots, S_m\}$ and $\{a_1, \ldots, a_m\}$ is the set of induced 
arrangements $a_i = a / S_i$, then the combined arrangement $a_1 \otimes \dots 
\otimes a_m = a$.
\end{lemma}
Given this identity, we can uniquely identify an arrangement 
$a$ in $G$ as the combination of its induced arrangements over a partition 
$\mathcal{P}$.

\subsection{Concrete Plans}
We now need to define what it means to move robots around a graph. First we 
will define two operators $\oplus$ and $\ominus$ which respectively add and 
remove robots from a given arrangement. Formally $\oplusg{G} : A_G \times R 
\times V(G) \rightarrow A_G$ is a mapping which satisfies \[ a \oplusg{G} (r,v) 
= b \] where
\[b(u) = 
\begin{cases}
 r &\text{~if~} u = v \\
 a(u) & \text{~otherwise}
\end{cases}
\]
Similarly $\ominusg{G} : A_G \times R \rightarrow A_G$ is a mapping which 
satisfies \[ a \ominusg{G} r = b \] where
\[
b(u) = 
\begin{cases} 
\Box & \text{~if~} a(u) = r \\
a(u) & \text{~otherwise}
\end{cases}
\]
We will omit the subscript $G$ when it is clear from the context.

We can now define a \e{plan-step} $s \in R \times E(G)$ in $G$ as a robot/edge 
pair $(r, u, v)$, representing the movement of $r$ along the edge from $u$ to 
$v$, with $u \neq v$. A plan-step is \e{applicable} to an arrangement $a \in 
A_G$ iff $a(u) = r$ and $a(v) = \Box$. In this case we can apply $s$ to $a$ to 
produce a new arrangement $b = s(a)$ where \[ s(a) = (a \ominus r) \oplus (r, 
v) \]

A \e{concrete plan} (or just \e{plan}) in $G$ from $a \in A_G$ to $b \in A_G$ is a
sequence of plan-steps $\seq{s_1, \ldots, s_l}$ such that there exist
arrangements $a_0, \ldots, a_l \in A_G$ with $s_i$ applicable to $a_{i-1}$ and
\begin{align*}
a_0 & =  a \\
a_l & =  b \\
a_i & =  s_i(a_{i-1}), ~\forall i : 0 < i \leq l 
\end{align*}

\begin{lemma} If $S$ is a subgraph of $G$ and $P$ is a plan in $S$ then $P$ 
is also a plan in $G$.
\end{lemma}

\begin{lemma} If $P$ is a plan in $G$ from $a$ to $b$ and $Q$ is a plan in $G$
from $b$ to $c$, then the concatenation of $P$ and $Q$, written $P.Q$ is a plan
in $G$ from $a$ to $c$.
\end{lemma}

\begin{lemma} Let $P \| Q$ denote the set of all interleavings of sequences $P$ and $Q$. Let $S_1$ and $S_2$ be disjoint subgraphs of $G$, $P_1$
be a plan on $S_1$ from $a_1$ to $b_1$ and $P_2$ 
be a plan on $S_2$ from $a_2$ to $b_2$, such that $a_1$ and $a_2$ are disjoint.
Any arbitrary interleaving $P \in P_1 \| P_2$ is a plan on $G$ from $a_1 \otimes a_2$ to $b_1 \otimes b_2$.
\end{lemma}

\subsection{Configurations}

Having defined the machinery for concrete plans, we now introduce 
abstraction. The key idea is that of a \emph{configuration} which is an
abstraction of arrangements. If the robots in a subgraph can be rearranged from
one arrangement to another, without any of the robots having to leave the
subgraph during the rearrangement, then those two arrangements can be treated as
equivalent. Configurations represent sets of equivalent arrangements in a subgraph. 
So, for example, in a stack subgraph a configuration is the set of all arrangements 
which have the same ordering of robots. An arrangement over an entire partitioned 
graph can be abstracted as the list of configurations it produces in each of its subgraphs.

Formally, we define a \e{configuration relation} $\simg{G}$ on graph $G$ as an
equivalence relation over $A_G$ such that $a \simg{G} b$ iff there exists plans
$P_{ab}$ and $P_{ba}$ in $G$ from $a$ to $b$ and from $b$ to $a$ respectively.

A \e{configuration} $c$ of $G$ is an equivalence class of $\simg{G}$. We write 
$c =\hspace{1ex}\conf{G}{a}$ to represent the equivalence class containing 
arrangement $a$. Let $C_G$ be the set of all configurations of $G$.

\begin{lemma}
If $a \simg{G} b$ then $range(a) = range(b)$
\end{lemma}
Given this identity, we can unambiguously define the range of a configuration
$c$ to be \[range(c) = range(a), \text{~for any~} a \in c\]

We can now extend our definitions of $\oplus$ and $\ominus$ to configurations.
If $c \in C_G$ is a configuration of $G$, $r \in R$ and $v \in V(G)$
then
\begin{align*}
c\oplusg{G} (r, v) & = \setof{\conf{G}{a \oplusg{G} (r, v)}}{a \in c, a(v) = \Box}\\
c\ominusg{G} (r, v) & = \setof{\conf{G}{a \ominusg{G} r}}{a \in c, a(v) = r}
\end{align*}
Note that $\oplus$ and $\ominus$ map configurations to sets of
configurations.\footnote{Astute readers will notice that $c\ominusg{G} (r, v)$
never contains more than one element, although it may be empty.}

Given a partition $\mathcal{P} = \{S_1, \ldots, S_m\}$ of $G$ and a 
corresponding set of configuration relations $\{\simg{S_1}, \ldots, 
\simg{S_m}\}$ we define a \e{configuration tuple} $\gamma$ of $R$ in $G$ as a 
tuple $(c_1, \ldots, c_m)$ where $\forall i : c_i \in C_{S_i}$, and 
\[\bigcup_{i=1 \ldots m} range(c_i) = R\]
\[range(c_i) \cap range(c_j) = \emptyset, ~\forall i,j : i \neq j\]
A configuration tuple represents the abstract state of all the robots in the
entire graph, in terms of the configurations of the individual subgraphs in the
partition. Given an arrangement $a$ of $G$ we can construct a corresponding configuration
tuple $\gamma(a) = (c_1, \ldots, c_m)$ where $c_i = \hspace{1ex}\conf{S_i}{a/S_i}$.
Conversely, if $a / S_i \in c_i$ for all $c_i$ in $\gamma$, then we write $a \in
\gamma$.

\begin{lemma} \label{lem:a} If $a$ and $b$ are arrangements in graph $G$ with
partition
$\{S_1, \ldots, S_m\}$ and $\gamma$ is a configuration tuple in $G$ with $a,b \in
\gamma$, then there exists a plan from $a$ to $b$ in $G$.
\end{lemma}
\begin{proof} 
For each $i = 1 \ldots m$, let $a_i = a / S_i$ and  $b_i = b /
S_i$. Now $a_i \in c_i$ and $b_i \in c_i$ so $a_i \simg{S_i} b_i$. Therefore
from the definition of $\sim$ there exists a plan $P_i$ from $a_i$ to $b_i$ in
$S_i$.

Let $P \in P_1 \| \dots \| P_m$. Since the $P_i$'s are plans on disjoint
subgraphs, $P$ is a plan from $a_1 \otimes \dots \otimes a_m = a$ to $b_1
\otimes \dots \otimes b_m = b$ as required.
\end{proof}

\subsection{Abstract Plans}

With configuration tuples as our abstract state representation, we can now 
define abstract plans, as sequences of subgraph transitions -- plan steps 
which move a robot from one subgraph to another. We will then prove the main 
result of this section, that an abstract plan for a problem exists if and only 
if a corresponding concrete plan exists. This will allow us later to prove the 
soundness and completeness of our subgraph planning algorithm. 

For the rest of this section we shall assume that our graph $G$ has a partition 
$\mathcal{P} = \{S_1, \ldots, S_m\}$ with corresponding configuration relations 
$\{\simg{S_1}, \ldots, \simg{S_m}\}$.

A \e{subgraph transition} (or just \emph{transition}) is a plan-step $s =
(r,u,v)$ such that $u \in S_x$, $v \in S_y$ and $S_x \neq S_y$. A transition 
$s = (r, u, v)$ is \e{applicable} to a configuration tuple $\gamma = (c_1,
\ldots, c_m)$ of $G$ if
\begin{align*}
   & c_x \ominusg{S_x} (r, u) \neq \emptyset, \text{~where~} u \in S_x, \\
\text{and~} & c_y \oplusg{S_y} (r, v) \neq \emptyset, \text{~where~} v \in S_y.
\end{align*}
That is, the robots in $S_x$ can be rearranged so that robot $r$ can leave via $u$
and the robots in $S_y$ can be rearranged so that $v$ is empty for $r$ to enter.

If transition $s = (r, u, v)$ is applicable to $\gamma = (c_1, \ldots, c_m)$
with $u \in S_x$ and $v \in S_y$ then we can
apply $s$ to $\gamma$ to compute a set $s(\gamma)$ of configuration-tuples
\[(c_1', \ldots c_m') \in s(\gamma)\]
if and only if
\begin{align*}
&c_x' \in c_x \ominusg{S_x} (r,u), \\
&c_y' \in c_y \oplusg{S_y} (r,v), \\
\text{and~}&c_z' = c_z, \text{~otherwise.}
\end{align*}

\begin{lemma}\label{lem:c}
If $a$ is an arrangement in $G$ with partition $\{S_1, \ldots, S_m\}$ and
transition $s = (r,u,v)$ is applicable to $a$ then $s$ is also applicable to
$\gamma(a)$, with
\[
\gamma(s(a)) \in s(\gamma(a))
\]
\end{lemma}
\begin{proof}
Let $S_x, S_y$ be disjoint subgraphs from the partition such that $u \in S_x$,
$v \in S_y$. Let $a_x = a / S_x$ and $a_y = a / S_y$. Let $\gamma(a) = (c_1,
\ldots, c_m)$. Now 
\begin{align*}
& a_x \in c_x \\
 & a_x(u) = r \\
\Rightarrow~  & c_x \ominus (r, u) \neq \emptyset.
\end{align*}
And similarly
\begin{align*}
& a_y \in c_y \\
 & a_y(v) = \Box \\
\Rightarrow~  & c_y \oplus (r, v) \neq \emptyset.
\end{align*}
Therefore $s$ is applicable in $\gamma(a)$.

Further, let $b = s(a)$ and $\gamma(b) = (c_1', \ldots, c_m')$. Now 
\begin{align*}
c_x' & = \conf{S_x}{b / S_x} \\
     & = \conf{S_x}{a_x \ominus r} \\
     & \in c_x \ominus (r,u)
\end{align*}
and
\begin{align*}
c_y' & = \conf{S_y}{b / S_y} \\
     & = \conf{S_y}{a_y \oplus (r,v)} \\
     & \in c_y \oplus (r,v)
\end{align*}
and
\begin{align*}
c_z' & = c_z.
\end{align*}
Therefore $\gamma(b) \in s(\gamma)$ as required.
\end{proof}

\begin{lemma}\label{lem:b}
If $s = (r,u,v)$ with $u,v \in S_x$ (i.e. $s$ is \emph{not} a transition) and $a$
is an arrangement in $G$ such that $s$ is applicable in $a$, 
then $\gamma(a) = \gamma(s(a))$.
\end{lemma}
\begin{proof}
Let $b = s(a)$. Let $a_i = a / S_i$ and $b_i = b / S_i$ for all $i = 1\ldots
m$. Let $\gamma(a) = (c_1, \ldots, c_m)$ and $\gamma(b) = (c_1', \ldots,
c_m')$.

Now the plan $P_x = \seq{s}$ is a plan from $a_x$ to $b_x$ in $S_x$, so $a_x
\sim{} b_x$ implying $c_x = c_x'$. For all other $z \neq x$, we have $a_z = b_z$
so $c_z = c_z'$. Therefore $\gamma(a) = \gamma(b)$ as required. 
\end{proof}

Now we can define an \e{abstract plan} $\Pi$ from arrangement $\alpha$ to $\beta$ in $G$ as a tuple $(\Gamma, \Sigma)$ where $\Gamma$ is a sequence of configuration tuples 
$\seq{\gamma_0, \ldots, \gamma_l}$  and $\Sigma$ is a sequence of plan steps 
$\seq{s_1, \ldots, s_l}$, such that
\begin{align*}
& \gamma_0 = \gamma(\alpha), \\
& \gamma_l = \gamma(\beta), \\
& s_i \text{~is applicable in~} \gamma_{i-1}, \\
\text{and~}& \gamma_i \in s(\gamma_{i-1}). 
\end{align*}

\begin{theorem} 
An abstract plan from $\alpha$ to $\beta$ in $G$ exists if and only if
there exists a corresponding concrete plan $P$ from $\alpha$ to $\beta$ in $G$.
\end{theorem}
\begin{proof} \textbf{Case} $(\Pi \Rightarrow P)$:

Let $\Pi = (\Gamma, \Sigma)$ be an abstract plan on $G$ from $\alpha$ to 
$\beta$, with $\Gamma = \seq{\gamma_0, \ldots, \gamma_l}$  and $\Sigma = 
\seq{s_1, \ldots, s_l}$. Let $\gamma_i = (c^i_0, \ldots, c^i_m)$.

We shall construct a concrete plan 
\[P = P_0 . \seq{s_1} . P_1 .\cdots . P_{l-1} . \seq{s_l} . P_l\] 
where each $P_i$ is a concrete plan from $a^i$ to $b^i$, satisfying
\begin{align*}
& a^0 = \alpha, \\
& b^l = \beta, \\
& a^i, b^i \in \gamma_i, \\
& s_{i+1} \text{~is~applicable~in~} b^i, \\
\text{and~}& a^{i+1} = s_{i+1}(b^i), ~\forall i = 0 \ldots l-1.
\end{align*}

\begin{proposition} $a^i$ and $b^i$ exist satisfying these conditions for all
$i = 1 \ldots \l$.
\end{proposition}
\begin{proof} by induction:

$a^0 = \alpha$ therefore $a^0$ exists.

Assume $a^i$ exists:

Let $s_{i+1} = (r, u, v)$ with $u \in S_x$ and $v \in S_y$. From the definition
of an abstract plan, $s_{i+1}$
is applicable in $\gamma_{i}$, and $\gamma_{i+1} = s_{i+1}(\gamma_i)$. Therefore
\begin{align*}
&c^{i+1}_x \in c^i_x \ominus (r, u) \neq \emptyset \\
\Rightarrow~&c^{i+1}_x \in \setof{\conf{}{a \ominus (r,u)}}{a(u) = r, a \in c^i_x} \neq
\emptyset \\
\Rightarrow~ &\exists a \in c^i_x : a(u) = r
\end{align*}
Set $b^i_x$ equal to this $a$. We now have
\begin{align*}
&c^{i+1}_x  = ~\conf{}{b^i_x \ominus (r,u)} \\
\Rightarrow~&b^i_x \ominus (r,u)  \in c^{i+1}_x
\end{align*}
Also
\begin{align*}
&c^{i+1}_y \in c^i_y \oplus (r, v) \neq \emptyset \\
\Rightarrow~&c^{i+1}_y \in \setof{\conf{}{a \oplus (r,v)}}{a(v) = \Box, a \in c^i_y} \neq \emptyset\\
\Rightarrow~&\exists a \in c^i_y : a(v) = \Box
\end{align*}
Set $b^i_y$ equal to this $a$. We now have
\begin{align*}
&c^{i+1}_y = ~\conf{}{b^i_y \oplus (r,v)}\\
\Rightarrow~&b^i_y \oplus (r,v) \in c^{i+1}_y
\end{align*}

Set $b^i_z = a^i / S_z$ for all other $z \notin \{x,y\}$

$b^i_j$ is now defined for every subgraph $S_j$ in the partition of $G$.
Therefore $b^i = b^i_1 \otimes \dots \otimes b^i_m$ exists and is an arrangement
in $G$. 

So if $a^i$ exists then $b^i$ also exists for all $i = 0 \ldots l-1$.

Now $s_{i+1}$ is applicable in $b^i$ since
\begin{align*}
b^i(u) &= b^i_x(u) = r\\
b^i(v) &= b^i_y(v) = \Box
\end{align*}

So $a^{i+1} = s_{i+1}(b^i)$ exists, and 

\begin{align*}
a^{i+1} / S_x & =   b^i_x \ominus r \in c^{i+1}_x \\
a^{i+1} / S_y & =   b^i_y \oplus (r,v)  \in c^{i+1}_y \\
a^{i+1} / S_z & =   b^i_z, \forall z \notin \{x, y\} \in c^i_z \\
              & \in c^{i+1}_z
\end{align*}
So
\[a^{i+1} \in \gamma_{i+1} \]

By induction, $a^i \in \gamma_i$ exists for all $i = 0 \ldots l$ and 
$b^i \in \gamma_i$ exists for all $i = 0 \ldots l-1$. Furthermore $b^l = \beta \in 
\gamma(\beta) = \gamma_l$, so $b^i$ exists for all $i$ for all $i = 0 \ldots 
l$, as required.
\end{proof}

\begin{proposition} A concrete plan $P_i$ from $a^i$ to $b^i$ exists, for $i =
0, \ldots, l$
\end{proposition}
\begin{proof} Since $a^i, b^i \in \gamma_i$ a plan $P_i$ must exist from 
$a_i$ to $b_i$, by Lemma~\ref{lem:a} above.
\end{proof}

\begin{proposition} $P$ is a concrete plan from $\alpha$ to $\beta$ in $G$.
\end{proposition}
\begin{proof} $P_i$ is a plan from $a^i$ to $b^i$ for all $i = 0, \ldots, l$. Furthermore
$a_{i+1} = s_{i+1}(b_i)$, so $\seq{s_{i+1}}$ is a plan from $b_i$ to $a_{i+1}$. 
Therefore by the concatenation of plans
\[P = P_0 . \seq{s_1} . \cdots . \seq{s_l} . P_l\]
is a plan in $G$ from $a_0 = \alpha$ to $b_l = \beta$, as required.
\end{proof}

\textbf{Case} $(P \Rightarrow \Pi)$:

Let $P = \seq{s_1, \ldots, s_L}$ be a concrete plan from $\alpha$ to $\beta$ in 
$G$. We wish to construct an abstract plan $\Pi = (\Gamma, \Sigma)$ from 
$\alpha$ to $\beta$ in $G$.

Let $T = \seq{t_0, \ldots, t_l}$ be an increasing sequence of integers with 
$t_0 = 0$ and $t_i = t$ iff $s_t$ is a subgraph transition. (Note: we are using capital $L$ to designate the length of the \emph{concrete} plan $P$ and lowercase $l$ to designate the number of transitions in that plan, which will be the length of the corresponding \emph{abstract} plan $\Pi$.) 

Now construct the sequence of arrangements $A = \seq{\alpha_0, \ldots,
\alpha_L}$ such that
\begin{align*}
\alpha_0 & = \alpha \\
\alpha_{i+1} & = s_{i+1}(\alpha_i), \forall i = 0 \ldots L-1
\end{align*}
and split $A$ into subsequences $A_0, \ldots, A_l$ such that
\[A_i = \seq{\alpha_{t_i}, \ldots, \alpha_{t_{i+1} - 1}}\]

Define $\gamma_i = \gamma(\alpha_{t_i}), \forall i = 0, \ldots, l$, $\Sigma =
\seq{\gamma_0, \ldots, \gamma_l}$ and $\Sigma = \seq{s_{t_1}, \ldots, s_{t_l}}$.

\begin{proposition} $\forall a : a \in A_i \Rightarrow a \in \gamma_i$
\end{proposition}
\begin{proof} by induction:

By definition,
\[\alpha_{t_i} \in \gamma(\alpha_{t_i}) = \gamma_i\]

Now assume $\alpha_t \in \gamma_i$ for $t = t_i + j, j < |A_i| - 1$. We need to
prove $\alpha_{t+1} \in \gamma_i$.

Let $s_{t+1} = (r, u, v)$. Since $t+1 \notin T$ we must have $u, v \in S_x$. So
using Lemma~\ref{lem:b} above
\begin{align*}
\gamma(\alpha_{t+1}) & = \gamma(s_{t+1}(\alpha_t)) \\
                     & = \gamma(\alpha_t) \\
                     & = \gamma_i.§
\end{align*}
Therefore, by induction
\[a \in \gamma_i, \forall a \in A_i\] as required.
\end{proof}

\begin{proposition} $\Pi = (\Gamma, \Sigma)$ is a valid abstract plan from
$\alpha$ to $\beta$.
\end{proposition}
\begin{proof}
First we check that the initial and final configuration-tuples contain $\alpha$
and $\beta$ respectively: 
\begin{align*}
& \gamma_0 = \gamma(\alpha_0) = \gamma(\alpha).
\end{align*}
and 
\begin{align*}
 & \beta \in A_l \\
\Rightarrow~& \beta \in \gamma_l, \\
\Rightarrow~& \gamma_l = \gamma(\beta).
\end{align*}

Now, for each $i = 0  \ldots l-1$ let $b^i = \alpha_{t_{i+1}-1}$ (i.e. the
final element of $A_i$), and let $b^i_z = b^i / S_z$ for $z=1 \ldots m$.

Let $s = s_{t_{i+1}} = (r, u, v)$ with $u \in S_x$ and $v \in S_y$. Now
$s$ is applicable in $b^i$ by the definition of $P$. Therefore, by
Lemma~\ref{lem:c} above, $s$ is applicable in $\gamma_i$ and 
\begin{align*}
\gamma_{i+1} &= \gamma(a^{i+1}) \\
            &= \gamma(s(b^i)) \\
            &\in s(\gamma(b^i)) = s(\gamma_i), \text{~as~required.}
\end{align*}
\end{proof} 

Therefore $\Pi$ is a valid abstract plan.
\end{proof}

This theorem is significant for our planning problem. It tells us that we do not
need to perform a search of all concrete plans. Instead, we need only search for
an abstract plan and then convert it into a concrete form. Such a search will
succeed if and only if a concrete plan exists.

\bibliography{jair06}
\bibliographystyle{theapa}

\end{document}